\documentclass[twoside]{article}

\usepackage[accepted]{aistats2026}

\usepackage[round]{natbib}
\renewcommand{\bibname}{References}
\renewcommand{\bibsection}{\subsubsection*{\bibname}}

\usepackage{microtype}
\usepackage{graphicx}
\usepackage{subfigure}
\usepackage{booktabs} 

\usepackage{hyperref}

\usepackage{wrapfig}

\usepackage{amsmath}
\usepackage{amssymb}
\usepackage{mathtools}
\usepackage{amsthm}

\usepackage[capitalize,noabbrev]{cleveref}

\usepackage[algo2e,ruled,linesnumbered]{algorithm2e}
\usepackage{multirow}

\usepackage{tikz}
\usetikzlibrary{arrows,calc,positioning}

\theoremstyle{plain}
\newtheorem{theorem}{Theorem}

\newtheorem{lemma}{Lemma}

\theoremstyle{definition}

\theoremstyle{remark}
\newtheorem{remark}{Remark}

\renewcommand{\t}{\text}
\newcommand{\op}[1]{\operatorname{#1}}
\newcommand{\C}[1]{{\mathcal{#1}}} \newcommand{\B}[1]{{\mathbb{#1}}} \newcommand{\BF}[1]{{\mathbf{#1}}}

\begin{document}

\twocolumn[
\aistatstitle{Multi-Armed Sampling Problem and the End of Exploration}

\aistatsauthor{ 
  Mohammad Pedramfar \And Siamak Ravanbakhsh 
}

\aistatsaddress{
  Mila - Quebec AI Institute, McGill University \\
  \texttt{\{mohammad.pedramfar, siamak.ravanbakhsh\}@mila.quebec}
}
]

\begin{abstract}
This paper introduces the framework of multi-armed sampling, which serves as the sampling counterpart to the optimization problem of multi-armed bandits. Our primary motivation is to rigorously examine the exploration-exploitation trade-off in the context of sampling. We systematically define plausible notions of regret for this framework and establish corresponding lower bounds. We then propose a simple algorithm that achieves near-optimal regret bounds. Our theoretical results suggest that, in contrast to optimization, \emph{sampling barely requires any exploration.} To further connect our findings with those of multi-armed bandits, we define a continuous family of problems and associated regret measures that smoothly interpolate and unify multi-armed sampling and multi-armed bandit problems using a temperature parameter. We believe that the multi-armed sampling framework and our findings in this setting can play a foundational role in the study of sampling, including recent neural samplers, much like the role of multi-armed bandits in reinforcement learning. In particular, our work sheds light on the role of exploration (or lack thereof) and the convergence properties of algorithms for entropy-regularized reinforcement learning, fine-tuning of pretrained models and reinforcement learning with human feedback (RLHF).
\end{abstract}

\section{INTRODUCTION}

Efficiently sampling from a computationally expensive target density is a key challenge in sciences with applications ranging from drug and material discovery to ecology and climate science. 
The problem involves sampling from a target distribution with density $p$ which is described by an unnormalized energy model $p(x) = \exp(-\C{E}(x))/Z$ with normalizing constant $Z = \int \exp(-\C{E}(x)) dx$.
To make the notation similar to reinforcement learning, we refer to $r(x) = -\C{E}(x)$ as the reward and assume that we have access to an unbiased estimate of $r$, but not the normalizing constant $Z$ or samples from $p$.
Obtaining estimates of the reward function often involves expensive and potentially noisy simulations or experiments, which motivates the development of efficient sampling techniques.
Examples in high-impact application areas include energy estimates based on approximate quantum mechanical calculations for new drug and material discovery, or large-scale hydrodynamical simulations in cosmology and astrophysics.
This problem is ubiquitous in Bayesian statistics and machine learning and has been an object of study for decades, with Monte-Carlo methods \citep{duane1987hybrid,roberts1996exponential,hoffman2014no,leimkuhler2014longtime,lemos2023ggns} recently being complemented by deep generative models \citep{albergo2019flow,noe2019boltzmann,gabrie2021efficient,midgley2022flow,akhoundsadegh2024iterated}.
Models such as GFlowNets \citep{bengio21_flow_networ_gener_model_non,malkin22_trajec,jain22_biolog_sequen_desig_gflow,jain23_multi_objec_gflow,kim24_genet_gflow_sampl_effic_molec_optim,jain23_gflow_ai}
also directly address this problem.

It is important to note that there are many problems that are not commonly formulated as sampling problems but are equivalent to one.
An important example is entropy-regularized RL.
It is a well-known fact that the optimization objective of entropy regularized RL could be formulated as the minimization of
KL-divergence between the policy and a certain form of target distribution function.
In particular, in the bandit case where there is only a single state and the Q-function is the same as the reward function, the target density becomes $\exp(r(x))/Z$, which is the sampling target that is the focus of our work.
(See~\cite{ziebart10_model_purpos_adapt_behav_princ,haarnoja2017reinforcement} and Theorem~\ref{thm:MAS_to_MAB}.)
In fact, it has been shown that, in many cases, the task of learning a GFlowNet may be formulated as entropy-regularized RL~\citep{tiapkin2024generative}.

An important application of entropy-regularized RL is in RLHF and fine-tuning of large pretrained models~\cite{jaques19_way_off_polic_batch_deep,stiennon20_learn,bai22_train_helpf_harml_assis_reinf,ouyang22_train,azar24_gener_theor_parad_under_learn_human_prefer}, such as LLMs (e.g.~\cite{rafailov23_direc_prefer_optim}) and diffusion models (e.g.~\cite{fan23_dpok}).
In LLMs, the equivalence of the DPO objective with a minimization of reverse-KL divergence could be seen 
in~\cite{rafailov23_direc_prefer_optim}.
At each time step, the learner generates a sample and observes the associated reward.
The DPO objective is equivalent to minimizing $D_{KL}(p_\theta || p)$, where $p_\theta$ is the model that is being tuned, and $p$ is the target density defined by
\begin{align*}
p(y | x) := \frac{1}{Z} p_{\op{ref}} (y | x) \exp(r(y | x)),
\end{align*}
where $x$ is the context, e.g., prompt, $y$ is a sample, $p(y | x)$ is the likelihood of generating $y$ given $x$ in the optimal target model, $p_{\op{ref}}$ is the pretrained model and $Z$ is the normalization constant.
This formulation is adopted in nearly all RLHF methods for LLM alignment \citep{bai22_train_helpf_harml_assis_reinf,ouyang22_train,touvron23_llama}.
\emph{It is important to note that, unlike the classical RL objective, the goal here is not to find ``the best'' sample.
Instead, the goal is to learn a policy that generates samples according to the target distribution.}

Despite all the progress in this field, the exploration-exploitation trade-off between sampling from identified modes and exploration to find new modes has not been addressed so far. The optimization counterpart of the sampling problem, which aims to minimize the energy, or maximize the reward corresponds to established areas of machine learning, namely Bayesian Optimization \citep[BO;][]{snoek2012practical} and Reinforcement Learning \citep{sutton2018reinforcement}. 
While the importance of exploration-exploitation trade-off and experimental design in the optimization paradigms is well-studied and understood, this trade-off in the context of sampling has gone unnoticed. 

Motivated by the defining role of multi-arm bandit algorithms~\citep{berry1985bandit,lattimore2020bandit}
in addressing this trade-off in the optimization setting, we set out to investigate an analogous problem of multi-arm sampling.
In this problem, at each timestep, the learner generates a sample and observes the reward, just as in a classical multi-armed bandit.
In order to rigorously address the exploration-exploitation tradeoff in this setting, we first need a uniform way to measure the performance of an algorithm.
In classical RL theory, this is done using the notion of regret.
However, as we mentioned earlier, the goal of sampling is not to maximize a scalar, but to match a target distribution.
Thus, we start with a systematic study of what ``regret'' could plausibly mean and how different notions of regret relate to each other.
Once the notions of regret are established, we investigate the greedy algorithm and an algorithm with minor exploration, to show that, at least in the multi-armed case, unlike RL, sampling barely requires any exploration.

While previous works consider similar problems, they are limited to special cases and either obtain suboptimal regret bounds or require significant exploration.
In particular, \citep{lu2018exploration} considered this problem with a notion of regret that we refer to as forward-KL policy-level regret and obtained $\tilde{O}(T^{1/2})$ regret bound using of UCB-type exploration.
Similarly, \citep{zhao25_logar_regret_onlin_kl_regul_reinf_learn} also considered UCB-type exploration, but obtained logarithmic reverse-KL policy-level regret.
Our work is the first to show that the performance of the greedy algorithm is nearly optimal, and the nature of the exploration-exploitation tradeoff is different from that of RL.
Among prior works \citep{xiong24_iterat_prefer_learn_human_feedb,zhao25_sharp_analy_kl_regul_contex_bandit_rlhf,zhao25_logar_regret_onlin_kl_regul_reinf_learn} consider a more general setting where the number of arms is not finite. In contrast, our focus on the more basic multi-armed setting enables us to provide a more comprehensive framework and analysis that also shows the redundancy of exploration for a range of plausible settings. For a comprehensive literature review, please see Appendix~\ref{sec:related}.

In addition to its potential role for understanding the exploration-exploitation trade-off in sampling, the multi-armed sampling problem can have direct applications of its own, in areas where multi-armed bandits are found useful; examples include online advertising, recommender systems, and adaptive routing. The benefit lies in avoiding repetition and providing diverse actions, choices, or recommendations. For example, suggesting different ads to the same visitor based on their click probability can be preferred to recommending the same "optimal" choice in every visit that ends with no interaction.

\subsection{Outline and Contributions}

The main contributions of our work is listed below.
\begin{enumerate}
\item We introduce the multi-armed sampling framework as a theoretical counterpart to multi-armed bandits, tailored for sampling tasks.
\item We systematically define various potential notions of regret specific to multi-armed sampling and discuss their relations with each other.
\item We establish lower bounds for all notions of regret discussed earlier.
\item We propose a simple algorithm, called Greedy Active Sampling with Warm-up, that achieves near-optimal regret bounds with respect to all notions of regret with little to no exploration. This is a surprising result, given that sampling reduces to optimization at zero temperature limit, and optimization does require explicit exploration.
\item We then generalize our regret bound to all algorithms that behave within the confidence bound (See Theorem~\ref{thm:AS2}), thus showing that a similar regret bound holds algorithms using optimism, Thompson sampling, or even pessimism.
\item We present a continuous family of problems that interpolates between multi-armed sampling and multi-armed bandit problems, offering a unified theoretical perspective.
\end{enumerate}

\begin{table*}[t] \small
\caption{Comparison of MAS algorithms} 
\label{tbl:main}
{\centering

\begin{tabular}{ | c | c | c | c | c | c | c | c | }
\hline
\multicolumn{3}{|c|}{Regret type} & Reference & Regret & Lower bound \\
\hline
\multirow{4}{*}{\rotatebox{90}{TV}}
& \multirow{2}*{$\C{P}$} & $\C{S}$
    & Theorem~\ref{thm:AS}
      & $\tilde{O}(T^{-1/2})$
        & $\Omega(T^{-1/2})$ \\
  \cline{3-6}
& & $\C{C}$
    & Theorem~\ref{thm:AS}
      & $\tilde{O}(T^{1/2})$
        & $\Omega(T^{1/2})$ \\
  \cline{2-6}
& \multirow{2}*{$\C{A}$}
  & $\C{S}$
    & Theorem~\ref{thm:AS}
      & $\tilde{O}(T^{-1/2})$
        & $\Omega(T^{-1/2})$ \\
  \cline{3-6}
& & $\C{C}$
    & Theorem~\ref{thm:AS}
      & $\tilde{O}(T^{1/2})$
        & $\Omega(T^{1/2})$ \\
\cline{1-6}
\multirow{8}{*}{\rotatebox{90}{r-KL}}
& \multirow{6}*{$\C{P}$} & \multirow{4}*{$\C{S}$}
    & Online Iterative GSHF~\cite{xiong24_iterat_prefer_learn_human_feedb}
      & $\tilde{O}(T^{-1/2})$
        & \multirow{4}*{$\Omega(T^{-1})$} \\
    \cline{4-5}
& & & Two-Stage Mixed-Policy Sampling~\cite{zhao25_sharp_analy_kl_regul_contex_bandit_rlhf}
      & $\tilde{O}(T^{-1})$
        & \\
    \cline{4-5}
& & & KL-UCB~\cite{zhao25_logar_regret_onlin_kl_regul_reinf_learn}
      & $\tilde{O}(T^{-1})$
        & \\
    \cline{4-5}
& & & Theorem~\ref{thm:AS}
      & $\tilde{O}(T^{-1})$
        & \\
  \cline{3-6}
& & \multirow{2}*{$\C{C}$}
    & KL-UCB~\cite{zhao25_logar_regret_onlin_kl_regul_reinf_learn}
      & $\tilde{O}(1)$
        & \multirow{2}*{$\Omega(1)$} \\
    \cline{4-5}
& & & Theorem~\ref{thm:AS}
      & $\tilde{O}(1)$
        & \\
  \cline{2-6}
& \multirow{2}*{$\C{A}$}
  & $\C{S}$
    & Theorem~\ref{thm:AS}
      & $\tilde{O}(T^{-1})$
        & $\Omega(T^{-1})$ \\
  \cline{3-6}
& & $\C{C}$
    & Theorem~\ref{thm:AS}
      & $\tilde{O}(1)$
        & $\Omega(1)$ \\
\cline{1-6}
\multirow{5}{*}{\rotatebox{90}{f-KL $\dagger$}}
& \multirow{3}*{$\C{P}$}
  & $\C{S}$
    & Theorem~\ref{thm:AS}
      & $\tilde{O}(T^{-1})$
        & $\Omega(T^{-1})$ \\
  \cline{3-6}
& & \multirow{2}*{$\C{C}$}
    & DAISEE~\cite{lu2018exploration}
      & $\tilde{O}(T^{1/2})$
        & \multirow{2}*{$\Omega(1)$} \\
  \cline{4-5}
& & & Theorem~\ref{thm:AS}
      & $\tilde{O}(1)$
        & \\
  \cline{2-6}
& \multirow{2}*{$\C{A}$}
  & $\C{S}$
    & Theorem~\ref{thm:AS}
      & $\tilde{O}(T^{-1})$
        & $\Omega(T^{-1})$ \\
  \cline{3-6}
& & $\C{C}$
    & Theorem~\ref{thm:AS}
      & $\tilde{O}(1)$
        & $\Omega(1)$ \\
\hline
\end{tabular}

}
{~\\
Lower-bounds that are matched by the regret bound of our simple algorithm that forgoes exploration.
Here, regret types vary along three axes: 
1) choice statistical distance; 
2) policy-level ($\C{P}$) vs. action-level ($\C{A}$), and 
3) simple ($\C{S}$) vs. cumulative ($\C{C}$) regret. \\
Note that the problem considered in~\cite{xiong24_iterat_prefer_learn_human_feedb,zhao25_sharp_analy_kl_regul_contex_bandit_rlhf,zhao25_logar_regret_onlin_kl_regul_reinf_learn} is KL-regularized bandit, in contextual setting with continuous arms. 
Moreover, the setting of~\cite{xiong24_iterat_prefer_learn_human_feedb} uses preference feedback model, which is slightly different than the ground-truth reward model used in other two papers. As we further discuss in Section~\ref{sec:MAS_to_MAB}, KL-regularized bandit is equivalent to sampling with policy level reverse-KL regret. 
Even though the setting we consider is finite-armed setting and therefore more limited, it allows us to provide a more comprehensive formulation and analysis.
As we see, the results of~\cite{xiong24_iterat_prefer_learn_human_feedb,lu2018exploration} are sub-optimal.
The algorithm proposed in~\cite{zhao25_sharp_analy_kl_regul_contex_bandit_rlhf} uses no exploration but the proposed algorithm is not a simple greedy algorithm and the analysis is limited to simple policy-level reverse-KL regret.
The algorithm in~\cite{zhao25_logar_regret_onlin_kl_regul_reinf_learn} use UCB-style exploration to obtain near optimal bounds for simple and cumulative policy-level reverse-KL regret.
We note that~\cite{lu2018exploration} also uses a form of UCB-style exploration, but the regret analysis provides sub-optimal bounds for cumulative policy-level forward-KL regret.
In our result, we consider two algorithms, one pure greedy and another with a small amount of exploration (much lower than that of optimism-based approaches), and show that the all forms of regret considered are near-optimal. 
In fact, as we show in Theorem~\ref{thm:AS2}, our main result also directly applies to KL-UCB and show that it also enjoys near-optimal regret bound for all forms of regret.
\\
$\dagger$: In the case of forward-KL distance, we assume that environment has bounded noise.
Also see Remark~\ref{rem:policy-fkl} about the notion of regret in this case.
}
\vspace{-.2in}
\end{table*}

\section{PROBLEM SETUP}\label{sec:problem_setup}

A \emph{multi-armed bandit sampling} problem, or simply \emph{multi-armed sampling} problem, is specified with a choice of regret (as we will discuss in Section~\ref{sec:regret}) and a multi-armed bandit environment.
The interaction between the policy and the environment is identical to that of multi-armed bandits; meaning that, at each time-step, the policy selects an action based on the history so far and observes a noisy sample of the reward. The core difference is the goal of the policy, encoded within the notion of regret. 

In this section, we focus on the environment.
Formally, the environment is a collection of distributions $\nu = (P_i)_{1 \leq i \leq k}$ where $k > 1$ is the number of actions.
The agent and the environment interact sequentially over $T$ time-steps.
In each time-step $1 \leq t \leq T$, the agent chooses an arm $\BF{a}_t \in [k] := \{1, \cdots, k\}$.
Then the environment samples a return value $\BF{x}_t$ from the distribution $P_{\BF{a}_t}$ and reveals $\BF{x}_t$ to the agent.
We define the \emph{unnormalized log-probability}, \emph{negative energy}, or \emph{reward} as $r := (r_i)_{i = 1}^k$ where $r_i$ is the mean of the distribution $P_i$, for $1 \leq i \leq k$.
The target distribution of $\nu$ is defined as $p := \op{softmax}(r)$, i.e.,
$p_i = e^{r_i}/\left( \sum_{j = 1}^k e^{r_j} \right)$,
for $1 \leq i \leq k$.
We also define
$\BF{n}_{t, i} := \sum_{t' = 1}^t \BF{1}_{\BF{a}_{t'} = i}$
and
$\BF{q}_{t, i} := \BF{n}_{t, i}/t$.
In other words, $\BF{q}_t = (\BF{q}_{t, 1}, \cdots, \BF{q}_{t, k})$ is the empirical probability distribution of the actions selected up to and including time-step $t$.
For $1 \leq t \leq T+1$, let $\BF{h}_t = (\BF{a}_1, \BF{x}_1, \cdots, \BF{a}_{t-1}, \BF{x}_{t-1})$ denote the history of actions and observation up to, but not including, time-step $t$ and let $\C{H}_t := ([k] \times \B{R})^t$ denote the space of all possible histories.
A policy $\pi$ is collection of functions $(\pi_t)_{t = 1}^T$ and a probability space $\Omega^\pi$.
Here the probability space $\Omega^\pi$ captures all the randomness in the policy.
Before taking the first action, random variable $\omega^\pi$ is sampled from $\Omega^\pi$ and remains constant throughout the ``game''.
Then, at time-step $t$, the function $\pi_t : \Omega^\pi \times \C{H}_t \to [k]$ maps $(\omega^\pi, \BF{h}_t) \in \Omega^\pi \times \C{H}_t$ to the next action $\BF{a}_t$.
We say a policy is \emph{deterministic} if $\Omega^\pi$ only contains a single point, or equivalently, if $\BF{a}_t$ is a deterministic function of the history.
We will use the notation $\hat{\pi}_t := (\B{P}(\BF{a}_t = 1 | \BF{h}_t), \cdots, \B{P}(\BF{a}_t = k | \BF{h}_t)) \in \B{R}^k$ to denote the distribution of $\BF{a}_t$ conditioned on the history.
We will also use $\omega^\nu \in \Omega^\nu := \B{R}^{T \times k}$ to denote a random matrix that captures all of the randomness of the environment $\nu$.
This matrix, which we call a \emph{realization of $\nu$}, specifies the reward at each arm and each time-step, i.e., if at time-step $t$ the arm $i$ is sampled, then the observed reward will be $(\omega^\nu)_{t, i}$.

For $\sigma > 0$ we define $\C{E}_{\op{SG}}^k(\sigma^2)$ as the class of environments $\nu = (P_i)_{1 \leq i \leq k}$ where each $P_i$ is $\sigma$-subgaussian.
Similarly, we define $\C{E}_{\op{B}}^k(\sigma)$ to be the class of environments where the absolute value of the noise of each arm is at most $\sigma$.
It follows from Hoeffding's lemma that $\C{E}_{\op{B}}^k(\sigma) \subseteq \C{E}_{\op{SG}}^k(\sigma^2)$.

We refer to Appendix~\ref{apx:noise} for a discussion on the necessity of noise and Appendix~\ref{apx:softmax} for the role of softmax in the problem formulation.

\section{NOTIONS OF REGRET}\label{sec:regret}

The goal of the agent is the multi-armed sampling problem is to match the target distribution $p$.
To formalize this idea, we need to decide exactly what distribution should be close to $p$ and how we measure this closeness.
We assume a \emph{performance distribution},  $f(\nu, \pi, \omega^\nu, \omega^\pi)$, representing some notion of agent's sampling distribution, which is a deterministic function of $(\nu, \pi, \omega^\nu, \omega^\pi)$, and a statistical distance $d$ between this performance distribution and the target distribution. The \emph{$(f, d)$-regret} is then defined as 
\begin{align*}
\C{R}^{f, d}(\pi, \nu, \omega^\pi, \omega^\nu) 
:= d\left( f(\pi, \nu, \omega^\pi, \omega^\nu), p \right).
\end{align*}
Hence, an algorithm that tries to minimize $(f, d)$-regret is attempting to make the distribution $f$ close to $p$ with respect to the statistical distance $d$.
We will simply use the notation $\C{R}^{f, d}(\pi, \nu)$ or $\C{R}^{f, d}$ when there is no ambiguity.
Since $f$ is allowed to depend on $\omega^\nu$ and $\omega^\pi$, it is a random variable and therefore $(f, d)$-regret is also a random variable.
The \emph{expected $(f, d)$-regret} of policy $\pi$ over environment $\nu$ is defined as
\begin{align*}
\B{E}\left[ \C{R}^{f, d} \right]
:= \B{E}_{\omega^\nu, \omega^\pi}\left[ d\left( f(\nu, \pi, \omega^\nu, \omega^\pi), p \right) \right],
\end{align*}
where the expectation is taken with respect to the randomness of $\omega^\nu$ and $\omega^\pi$.
More generally, if $l \geq 1$ and $f = (f_1, \cdots, f_l)$ is a family of performance distributions as above, we define $(f, d)$-regret and \emph{expected $(f, d)$-regret} as their corresponding summation
\begin{align*}
\C{R}^{f, d}
&:= \sum_{s = 1}^l d\left( f_s(\nu, \pi, \omega^\nu, \omega^\pi), p \right), \\\B{E}\left[ \C{R}^{f, d} \right]
&:= \B{E}\left[ \sum_{s = 1}^l d\left( f_s(\nu, \pi, \omega^\nu, \omega^\pi), p \right) \right].
\end{align*}
This is a general notion meant to capture many plausible formalizations of the idea of regret in multi-armed sampling problems.
In the remainder of this section, we consider some possible options for $f$ and $d$.

\subsection{ The statistical distance }
We need a notion of distance in order to compare distributions.
In this work, we mostly focus on three notions: total variation distance, reverse KL distance, and forward KL distance.
Specifically, given a distribution $\hat{p}$ and the target distribution $p$, we define
$d^{\op{r-KL}}(\hat{p}, p) 
:= D_{\op{KL}}(\hat{p} || p) = \sum_{i = 1}^k \hat{p}_i \log\left(\hat{p}_i/p_i\right)$
and
$d^{\op{f-KL}}(\hat{p}, p) 
:= D_{\op{KL}}(p || \hat{p}) = \sum_{i = 1}^k p_i \log\left( p_i/\hat{p}_i\right)$,
where we use $\log$ to denote the natural logarithm and
$d^{\op{TV}}(\hat{p}, p) 
:= \sup_{A \subseteq [k]} | \hat{p}(A) - p(A) | 
= \frac{1}{2} \sum_{i = 1}^k | \hat{p}_i - p_i |$.
We simply use $d$ and drop the superscript when the type of distance is either irrelevant or clear from the context.

\subsection{Action-level vs. policy-level regret}\label{sec:action-vs-policy-regret}

Next we need to decide the performance function, i.e., the distribution that is being compared to the target distribution.

When the performance function is $\hat{\pi}_t$, we refer to the regret as the \emph{policy-level regret} denoted by $\C{PR} := d\left( \hat{\pi}_t, p \right)$.
Minimizing this notion of regret means that we require $\hat{\pi}_t$ to remain close to $p$ in every time-step.

This is a strong requirement and there are useful algorithms for which this regret is not minimized.
For example, one could use Metropolis-Hasting algorithm to sample from a target $p$ using a proposal distribution for local moves, which means $\hat{\pi}_t$ will also have a local support.
In fact, the chain distribution in Markov Chain Monte Carlo methods is the average of $\hat{\pi}_t$ over the randomness of the policy, i.e., $\B{E}_{\omega^\pi}[\hat{\pi}_t]$.
Thus, in order to handle such scenarios, we define \emph{environment policy-level regret} by $\C{PR}^{\op{env}} := d\left( \B{E}_{\omega^\pi}[\hat{\pi}_t], p \right)$.
In this case, the notion of regret is no longer defined for any specific run of the algorithm, but for a all runs of the algorithm in a given realization of the  environment, i.e., we average over all values of $\omega^\pi$ while keeping $\omega^\nu$ fixed.
Thus, the two types of expected policy-level regrets are given by
\begin{align*}
\B{E}\left[ \C{PR} \right] 
&= \B{E}_{\omega^\nu, \omega^\pi}\left[ d\left( \hat{\pi}_t, p \right) \right], \\
\B{E}\left[ \C{PR}^{\op{env}} \right] 
&= \B{E}_{\omega^\nu}\left[ d\left( \B{E}_{\omega^\pi}\left[ \hat{\pi}_t \right] , p \right) \right].
\end{align*}

One problem with policy-level regret is that it is not always meaningful.
For example, assume the environment has two arms and consider the policy that chooses arm 1 in odd time steps and arm 2 in even time steps.
If the target distribution is $p = (1/2, 1/2)$, then the policy is behaving well with respect to the target, but its policy-level regret is the statistical distance between $p$ and $(1,0)$ or $(0,1)$, which is high.

When the performance function is $\BF{q}_t$, we refer to the regret $\C{AR}(\pi) := d\left( \BF{q}_t, p \right)$ as the \emph{action-level regret}.
Note that $\BF{q}_t$ is a deterministic function of $\BF{h}_{t+1}$ and therefore it is a random variable.
Thus the \emph{expected action-level regret} is given by $\B{E}_{\omega^\nu, \omega^\pi}\left[ d\left( \BF{q}_t, p \right) \right]$.
Similar to above, we define \emph{environment action-level regret} by $\C{AR}^{\op{env}} := d\left( \B{E}_{\omega^\pi}[\BF{q}_t], p \right)$.
Thus, the two types of expected action-level regrets are given by
\begin{align*}
\B{E}\left[ \C{AR} \right] 
&= \B{E}_{\omega^\nu, \omega^\pi}\left[ d\left( \BF{q}_t, p \right) \right],\\
\B{E}\left[ \C{AR}^{\op{env}} \right] 
&= \B{E}_{\omega^\nu}\left[ d\left( \B{E}_{\omega^\pi}\left[ \BF{q}_t \right] , p \right) \right].
\end{align*}

\subsection{Simple vs. cumulative regret}

In bandit literature, simple regret refers to the cost of the action at the last time-step relative to the optimal action.
Inspired by this idea, we define the \emph{simple action-level regret} and \emph{simple policy-level regret} of policy $\pi$, up to time-step $T$ as
\begin{align*}
\C{SAR}_T(\pi) := d(\BF{q}_T, p)
,\quad
\C{SPR}_T(\pi) := d(\hat{\pi}_T, p).
\end{align*}
We use superscript to specify the type of distance, e.g., we use $\C{SAR}^{\op{TV}}_T$ to denote simple action-level total variation regret.

In the bandit setting, selecting suboptimal arms at any time-step other than the last does not contribute to simple regret.
Hence algorithms that are designed to minimize simple regret are free to explore as much as possible and do not need to consider the exploration-exploitation trade-off.
This is why optimizing for simple regret in bandit settings are sometimes referred to as pure exploration problems.
This is in contrast with the problem of minimizing cumulative regret which takes the entire history of actions into account.

In our setting, minimizing simple policy-level regret could be considered a pure exploration problem.
For example, consider the algorithm that samples each arm $(T-1)/k$ times deterministically and then samples the last arm based on the sample mean of the observations.
Such an algorithm would have low simple policy-level regret at time $T$ but not earlier.

On the other hand, minimizing simple action-level regret is not pure exploration problem as it takes the history of actions into account.
However, it is still mostly focused on the end result and not the process.
For example, if there are only two arms, $p = (1/2, 1/2)$ and $T = 1000$, the algorithm that select the first arm for the first 500 steps and the second arm for the remaining steps achieves zero simple action-level regret.

In order to formalize and measure the performance of algorithms its history, we define \emph{cumulative regret} as the sum of simple regret over all time-steps.
\begin{align*}
\C{CAR}_T(\pi) := \sum_{t = 1}^T d(\BF{q}_t, p)
,\quad
\C{CPR}_T(\pi) := \sum_{t = 1}^T d(\hat{\pi}_t, p).
\end{align*}

\begin{remark}\label{rem:policy-fkl}
When considering forward-KL regret, one issue is that if the output of the performance function does not have full support, then the regret will be infinity.
This is particularly problematic if we consider policy-level forward-KL regret.
If at any time-step, the action taken by the policy is deterministic or it is stochastic but not supported on all arms, then the simple regret at that time-step and the cumulative regret for any of the following time-steps is infinite.
Thus, in Algorithm~\ref{alg:AS}, we only consider simple policy-level forward-KL regret at times $t > Mk$ and we define 
$\C{CPR}_T^{\op{f-KL}}(\pi) := \sum_{t = Mk + 1}^T d^{\op{f-KL}}(\hat{\pi}_t, p)$.
Note that~\cite{lu2018exploration} which considered this notion of regret also skips the first $k$ time-steps in their algorithm, i.e., DAISEE, when computing their regret.
\end{remark}

Given all these possible notions of regret, it is natural to ask if one is preferable to others.
The above remark suggests not using policy-level forward-KL regret.
We postpone a complete answer to this question until Section~\ref{sec:MAS_to_MAB}.

\section{RELATION BETWEEN DIFFERENT NOTIONS OF REGRET}\label{sec:relation}

\paragraph{Statistical distance}
We start by comparing the effect of statistical distance over different notions of regret.
The core idea here is to use Pinsker's inequality and inverse Pinsker's inequality to relate total variation and KL-divergences.

\begin{theorem}\label{thm:KL-vs-TV}
For any performance distribution $f$, we have
\begin{align*}
2 \left( \C{R}^{f, d^{\op{TV}}} \right)^2
&\leq \C{R}^{f, d^{\op{r-KL}}}
\leq \frac{2}{\alpha} \left( \C{R}^{f, d^{\op{TV}}} \right)^2,\\
2 \left( \C{R}^{f, d^{\op{TV}}} \right)^2
&\leq \C{R}^{f, d^{\op{f-KL}}}
\leq \frac{2}{\alpha'} \left( \C{R}^{f, d^{\op{TV}}} \right)^2,
\end{align*}
where $\alpha = \min_i p_i$ and $\alpha' = \min_i f_i$.
Moreover, we also have
\begin{align*}
2 \B{E}\left[ \C{R}^{f, d^{\op{TV}}} \right]^2
&\leq \min \left\{ 
  \B{E}\left[ \C{R}^{f, d^{\op{r-KL}}} \right],
  \B{E}\left[ \C{R}^{f, d^{\op{f-KL}}} \right]
\right\}
\end{align*}
\end{theorem}

See Appendix~\ref{apx:thm:KL-vs-TV} for the proof.

\paragraph{Regret vs environment regret}

In Section~\ref{sec:action-vs-policy-regret}, we mentioned that environment regret is weaker than regret.
This claim is formalized in the following theorem.

\begin{theorem}
If the map $q \mapsto d(q, p)$ is convex, then we have
\begin{align*}
\B{E}\left[ \C{PR} \right] 
&\geq
\B{E}\left[ \C{PR}^{\op{env}} \right],
\quad
\B{E}\left[ \C{AR} \right] 
\geq
\B{E}\left[ \C{AR}^{\op{env}} \right],
\end{align*}
for both simple and cumulative regret.
In particular, this statement holds for $d \in \{ d^{\op{TV}}, d^{\op{r-KL}}, d^{\op{f-KL}} \}$.
\end{theorem}
\begin{proof}
Since $q \mapsto d(q, p)$ is convex, we may use Jensen's inequality to see that
\begin{align*}
\B{E}_{\omega^\nu, \omega^\pi}\left[ d\left( \hat{\pi}_t , p \right) \right]
&\geq
\B{E}_{\omega^\nu}\left[ d\left( \B{E}_{\omega^\pi}\left[ \hat{\pi}_t \right] , p \right) \right],
\end{align*}
for any $1 \leq t \leq T$.
The claim follows from the definition of the notions of regret mentioned.
The proof for action-level regret is obtained similarly by replacing $\hat{\pi}_t$ with $\BF{q}_t$.
\end{proof}

\paragraph{Action-level vs policy-level regret} 

Finally, we consider the relation between action-level vs policy-level regrets.
As mentioned in Section~\ref{sec:action-vs-policy-regret}, a deterministic algorithm may have low action-level regret but high policy-level regret.
Thus, in general, we can not bound policy-level regret using action-level regret.
Instead, we try to bound action-level regret by policy-level regret.
Specifically, we show that that simple action-level regret is bounded by cumulative policy-level regret up to an additive term.
We will use this idea in the proof of Theorem~\ref{thm:AS} to bound the action-level regret of Algorithm~\ref{alg:AS}.

\begin{theorem}\label{thm:action-level-bounded-by-policy-level}
If the map $q \mapsto d(q, p)$ is convex, then we have
\begin{align*}
\C{SAR}_T - \frac{1}{T} \C{CPR}_T
&\leq d\left( \BF{q}_T, p \right) 
  - d\left( \frac{1}{T} \left( \sum_{t=1}^{T} \hat{\pi}_t \right), p \right).
\end{align*}
\end{theorem}
\begin{proof}
Since $d$ is convex in its first argument, we have
\begin{align*}
\C{SAR}_T &- \frac{1}{T} \C{CPR}_T
= d\left( \BF{q}_T, p \right) - \frac{1}{T} \sum_{t=1}^{T} d\left( \hat{\pi}_t , p \right)  \\
&\leq d\left( \BF{q}_T, p \right) - d\left( \frac{1}{T} \left( \sum_{t=1}^{T} \hat{\pi}_t \right), p \right).
\qedhere
\end{align*}
\end{proof}

\begin{remark}
Note that the expression in the right-hand side is not necessarily negligable.
For example, in the case of a two-armed sampling problem with target distribution $(1/2, 1/2)$ and a policy that selects uniformly randomly at each time step, all types of policy-level regret are zero.
However, simple action-level total-variation regret would be the difference between the average number of times the first arm is selected and $1/2$, which has a variance of $O(T^{-1/2})$, thus contributing $O(T^{-1/2})$ to the expected regret.
\end{remark}

\section{LOWER BOUNDS}\label{sec:lowerbounds}
In this section, we discuss lower bounds for regret bounds of multi-armed sampling problems.

First, we define a class of performance functions that do not have an unreasonable amount of access to the environment.
Specifically, we say a performance function is \emph{environment agnostic} if it can be written as a deterministic function of $(\pi, \omega^\pi, \omega^\nu)$.
Such a performance function has no direct access to the distribution of the environment and instead only depends on the realization of it, i.e., $\omega^\nu$.
As such, if two environments are sufficiently similar so that they produce similar realizations with high probability, then such a performance function would have a difficult time distinguishing between the environments.
This is the core idea we use to obtain lower bounds. 
Note that \emph{all performance functions considered in this paper are environment-agnostic.}

Next, we state the main lemma which provides a lower bound on the sum of the regret of a single policy in two different environments based on the difference between the environments.
To put it simply, if the environments are sufficiently different, then a policy that performs well in one environment can not perform well in the other.
See Appendix~\ref{apx:lem:lower-bound} for the proof.

\begin{lemma}\label{lem:lower-bound}
Let $\nu$ and $\nu'$ denote multi-armed sampling problems with reward functions $r : [k] \to \B{R}$ and $r' : [k] \to \B{R}$ and unit normal noise and let $\epsilon := \max_i | r_i - r'_i |$, $p := \op{softmax}(r)$ and $p' := \op{softmax}(r')$.
Then, for any policy $\pi$, statistical distance $d$, and environment agnostic performance function $f$ as defined above, we have
\begin{align*}
\B{E}\left[ \C{R}^{f, d}(\pi, \nu) \right]
&+ \B{E}\left[ \C{R}^{f, d}(\pi, \nu') \right] \\
&\geq \frac{1}{2} \exp \left( - \frac{1}{2} T \epsilon^2 - \epsilon \sqrt{2 T} \right) d'(p, p'),
\end{align*}
where $d'(p, p') := \inf_q d(q, p) + d(q, p')$.
\end{lemma}

Using the above lemma, we obtain lower bounds for all forms of regret discussed earlier.

\begin{theorem}\label{thm:lower-bound}
Let $k > 1$ and $T > 4$ be the number of arms and the horizon respectively, and let $f$ be any environment-agnostic performance function as defined above.
Then, for any policy $\pi$, we have
\begin{align*}
\sup_{\nu \in \C{E}_{\op{SG}}^k(1)} \B{E}\left[ \C{R}^{f, d^{\op{TV}}}(\pi, \nu) \right]
&\geq \frac{1}{126 \sqrt{T k}},\\
\sup_{\nu \in \C{E}_{\op{SG}}^k(1)} \B{E}\left[ \C{R}^{f, d}(\pi, \nu) \right]
&\geq \frac{1}{567 T k},
\end{align*}
for $d \in \{ d^{\op{r-KL}}, d^{\op{f-KL}} \}$.
Thus, these lower bounds apply to all notions of simple regret considered in this work.
Moreover, by multiplying the bounds by $T$, we obtain corresponding lower bounds for cumulative regret.
\end{theorem}

See Appendix~\ref{apx:thm:lower-bound} for the proof.

\section{AN OPTIMAL ALGORITHM}\label{sec:AS}

In this section, we introduce the Greedy Active Sampling with Warm-up (GASW) algorithm, which is a sampling version of the classical Explore-Then-Commit algorithm in multi-armed bandit problems.
In the case where there is no exploration, i.e. exploration factor $M$ is equal to 1, \emph{the algorithm becomes greedy} and we simply refer to it as Greedy Active Sampling (GAS).
As seen in Algorithm~\ref{alg:AS}, GAS performs random exploration for $MK$ steps, and then it samples each arm according to its reward estimate. The algorithm then continuously updates these estimates during sampling.

\begin{algorithm2e} \SetKwInOut{Input}{Input}\DontPrintSemicolon
\caption{Greedy Active Sampling (with Warm-up) (GAS / GASW)}
\label{alg:AS}
\small
\Input{Horizon $T$, number of arms $k$, and exploration factor $1 \leq M \leq T/k$}
\For{$t = 1$ to $T$}{
  \eIf{$t \leq Mk$}{
    Play $\BF{a}_t \gets (t \mod k) + 1$ \;
  }{
    Sample $1 \leq \BF{a}_t \leq k$ according to $\hat{\BF{p}}_t$ and play $\BF{a}_t$ \;
  }
  \For{$1 \leq i \leq k$}{
    $\hat{\BF{r}}_{t, i} \gets \frac{1}{\BF{n}_{t, i}} \sum_{t'=1}^t \BF{x}_{t'} \BF{1}_{\BF{a}_{t'} = i}$ \;
  }
  $\hat{\BF{p}}_{t+1} \gets \op{softmax}(\hat{\BF{r}}_t)$ \;
}
\end{algorithm2e}

The following theorem demonstrates that, unlike bandits, there is barely any need for exploration in multi-armed sampling problems.
Specifically, we see that with a negligible exploration of $O(\log(T))$, we obtain near-optimal regret bounds that differ from our lower bounds up to logarithmic terms.
Moreover, if there is no exploration at all, then the regret bounds would be multiplied by a factor that grows faster than $\log(T)$, but slower than $T^{\varepsilon}$, for any $\varepsilon > 0$.
Note that with the same amount of exploration, the Explore-Then-Commit algorithms suffer a linear regret in multi-armed bandit problems.

In the following theorem, we use the notation $\hat{\C{O}}^{1/2}$ to hide $\frac{1}{2}$-quasi-polynomial terms.
\footnote{
Specifically, we say $f(T) = \hat{\C{O}}^{1/2}(g(T))$ if $f(T) = \C{O}(g(T) h(T))$ for some $h$ that satisfies $h(T) = \exp\left( \C{O}\left(\log(T)^{1/2}\right) \right)$, which implies that $h(T) = \C{O}(T^\varepsilon)$ for all $\varepsilon > 0$.
}

\begin{theorem}\label{thm:AS}
In Algorithm~\ref{alg:AS}, if $\nu \in \C{E}_{\op{SG}}^k(\sigma^2)$ and $M \geq 36 \sigma^2 \log(T)$ and $d = d^{\op{TV}}$, then the regrets
$\C{SAR}_T^{d}$,
$\C{SAR}_T^{d, \op{env}}$,
$\C{SPR}_T^{d}$,
and $\C{SPR}_T^{d, \op{env}}$
are all bounded by $\tilde{\C{O}}(T^{-1/2})$ while the corresponding cumulative regrets are bounded by $\tilde{\C{O}}(T^{1/2})$.
On the other hand, if $d = d^{\op{r-KL}}$ or $d = d^{\op{f-KL}}$ and we also have $\nu \in \C{E}_{\op{B}}^k(\sigma)$, then the simple regrets are bounded by $\tilde{\C{O}}(T^{-1})$ while the corresponding cumulative regrets are $\tilde{\C{O}}(1)$. \\
Moreover, if we have no exploration, i.e., $M = 1$, then the same bounds hold if we replace $\tilde{\C{O}}$ with $\hat{\C{O}}^{1/2}$.
\end{theorem}

In fact, a more general result holds.

\begin{theorem}\label{thm:AS2}
Let $C \geq 0$ be a hyperparameter and let $\C{A}$ be an algorithm that behaves similar to Algorithm~\ref{alg:AS}, with the only difference being that $\C{A}$ selects $\hat{\BF{r}}_{t, i}$ such that
$\left| \hat{\BF{r}}_{t, i} - \bar{\BF{x}}_{t, i} \right| \leq 2 C \sigma \sqrt{\frac{\log(T) }{\BF{n}_{t, i}}}$,
where $\bar{\BF{x}}_{t, i} := \frac{1}{\BF{n}_{t, i}} \sum_{t'=1}^t \BF{x}_{t'} \BF{1}_{\BF{a}_{t'} = i}$.
Then regret bounds similar to those of Theorem~\ref{thm:AS} hold for $\C{A}$ as well.
\end{theorem}

See Appendix~\ref{apx:thm:AS} for the exact expressions of the regret bounds and the proof for both theorems.

\begin{figure*}
\centering
\hbox{
\includegraphics[width=\textwidth/3]{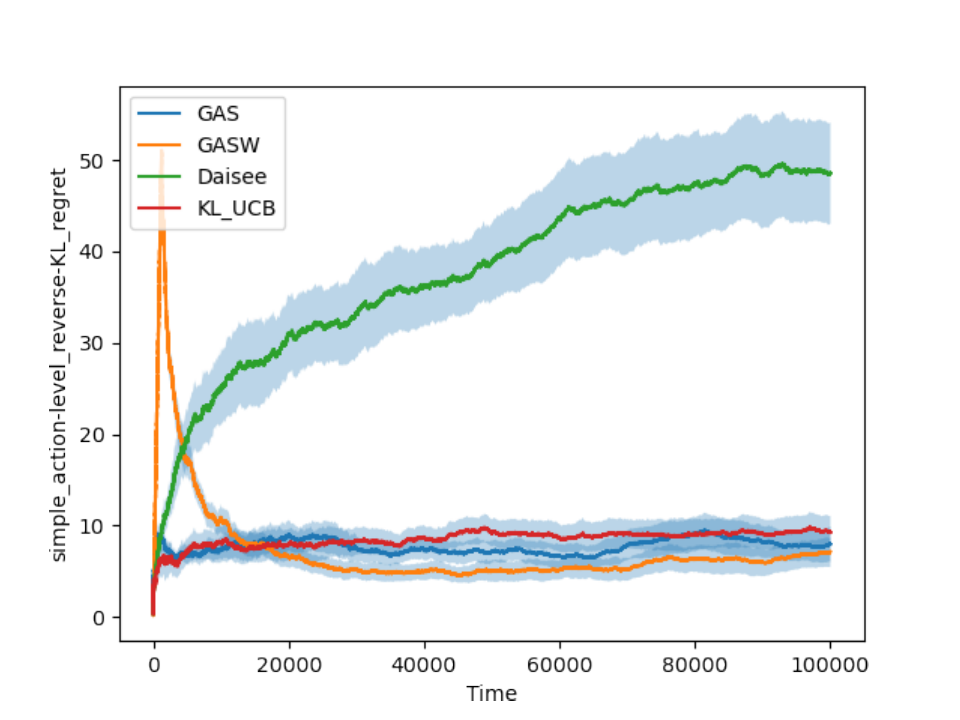}

\includegraphics[width=\textwidth/3]{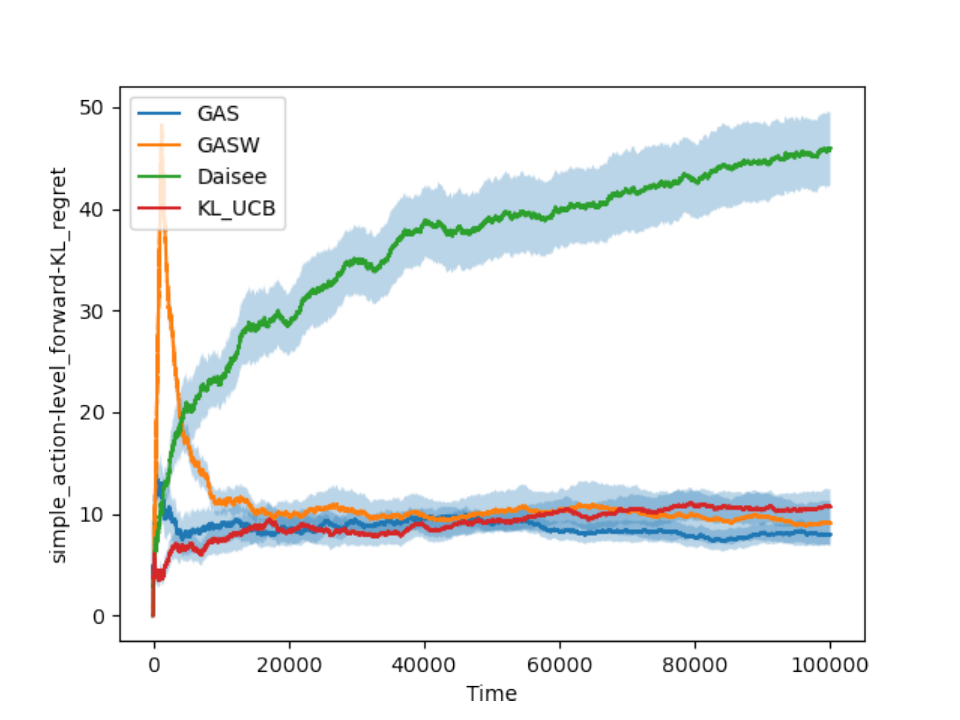}

\includegraphics[width=\textwidth/3]{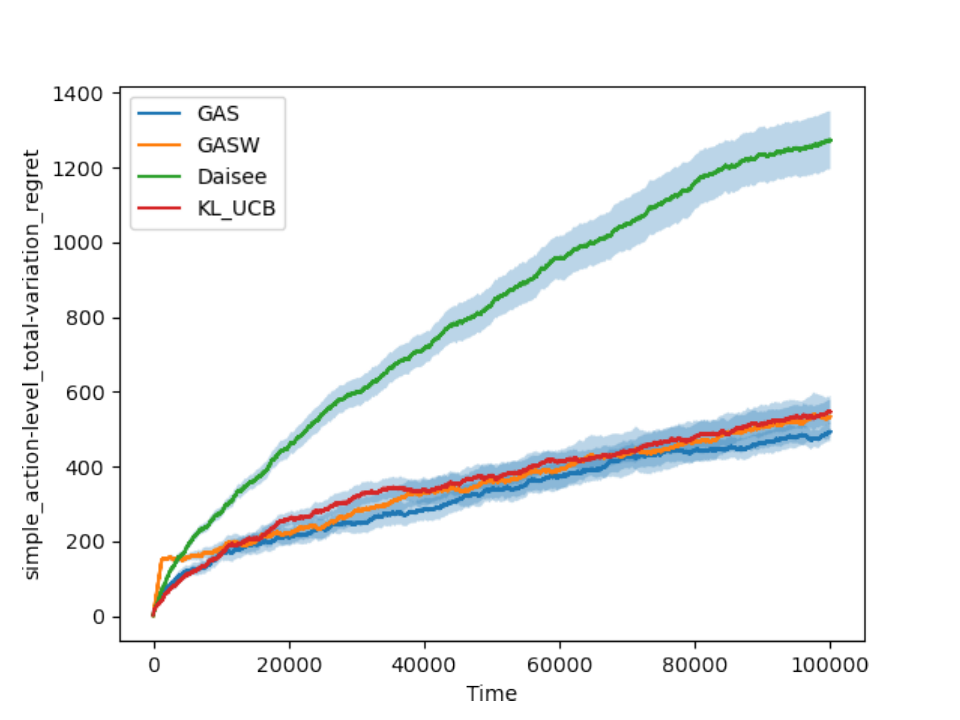}
}

\caption{
Comparison between regret, specifically $T.\C{SAR}_T$ 
, of GAS, GSAW with logarithmic exploration, DAISEE, and KL-UCB with different statistical distances.
From left to right, the distances are reverse-KL, forward-KL and total variation.
Experiments for other notions are regret are included in the appendix.
The initial increase in regret in the orange lines corresponds to the exploration phase (and its effect on action history).
Environments have unit normal noise and the reward are sampled from $[0, 1]$.
The average over 10 runs are shown.
See Appendix~\ref{apx:experiments} for details.
}
\label{fig:main}
\end{figure*}

Note that, when $C = 0$, the algorithm becomes greedy (with warm-up) and Theorem~\ref{thm:AS2} reduces to Theorem~\ref{thm:AS}.
On the other hand, Theorem~\ref{thm:AS2} covers UCB-type exploration (as in KL-UCB~\cite{zhao25_logar_regret_onlin_kl_regul_reinf_learn}) by choosing 
$\hat{\BF{r}}_{t, i} = \bar{\BF{x}}_{t, i} + 2 C \sigma \sqrt{\frac{\log(T) }{\BF{n}_{t, i}}}$.
Similarly, this result also covers pessimism, when we choose
$\hat{\BF{r}}_{t, i} = \bar{\BF{x}}_{t, i} - 2 C \sigma \sqrt{\frac{\log(T) }{\BF{n}_{t, i}}}$.
One approach to obtain regret bounds for Thompson Sampling (TS) is to use the well-known connection between TS and UCB to convert the regret bounds of UCB into regret bounds for TS. \citep{russo14_learn_optim_poster_sampl,russo18_tutor_thomp_sampl}.
Thus, a similar bound also applies to many forms of Thompson sampling algorithms.

\emph{
These regret bounds clearly demonstrate that the nature of exploration-exploitation tradeoff is different in MAB and MAS.
Specifically, when using optimism, Thompson sampling, or even pessimism, we still obtain near optimal regret bounds.
}

\section{FROM SAMPLING TO OPTIMIZATION}\label{sec:MAS_to_MAB}

Consider the multi-armed sampling problem with arm distribution $\nu$ as defined in Section~\ref{sec:problem_setup}.
We define $\C{R}_T^{\op{MAB}} := T \max_{1 \leq i \leq k} r_i - \sum_{i = 1}^k \BF{n}_{T, i} r_i$, i.e., the regret if the reward are distributed according to $\nu$

Next we let $\beta \in [0, \infty]$ and define 
$p^\beta := \op{softmax}(\beta r)$
for $\beta < \infty$ and 
$p^\infty := \lim_{\beta \to \infty} p^\beta$ 
where the limit is taken as vectors in $\B{R}^k$. $\beta$ is often referred to as the inverse-temperature.
It follows that $p^\infty$ is simply the distribution that assigns the same weight to all optimal arms $\op{argmax}_{1 \leq i \leq k} r_i$ while assigning zero weight to any other arm.
In other words, sampling according to $p^\infty$ is the same as maximizing reward.
Thus, it is natural to expect that, at the limit of $\beta \to \infty$, the multi-armed sampling problem becomes similar to the multi-armed bandit problem.
The following theorem makes this connection precise for the notion of reverse-KL regret.

First we define the \textit{multi-armed sampling action-level $\beta$-regret, simple policy level $\beta$-regret and policy level $\beta$-regret}, as
\footnote{Here we do not consider the cumulative version of action-level regret, since the simple version already captures the main ideas.}
\begin{align}\label{eqn:beta-regret}
\C{AR}_T^\beta &:= 
\begin{cases}
\frac{T}{\beta} d^{\op{r-KL}}(\BF{q}_T, p^\beta)  & \t{if } \beta < \infty, \\
\C{R}_T^{\op{MAB}}  & \t{if } \beta = \infty.
\end{cases},\\
\C{SPR}_T^\beta &:= 
\begin{cases}
\frac{T}{\beta} d^{\op{r-KL}}(\hat{\pi}_T, p^\beta)  & \t{if } \beta < \infty, \\
\C{R}_T^{\op{MAB}}  & \t{if } \beta = \infty.
\end{cases}
\end{align}
and $\C{PR}_T^\beta := \sum_{t = 1}^T \C{SPR}_t^\beta$, respectively.

\begin{theorem}\label{thm:MAS_to_MAB}
We have
\begin{align*}
\C{AR}_T^\beta
&= \C{R}_T^{\op{MAB}} - \frac{T}{\beta}H(\BF{q}_T) - \frac{T}{\beta}\log(p^\beta_*),\\
\C{PR}_T^\beta
&= \sum_{t = 1}^T \B{E} \left[ \C{SR}_t^{\op{MAB}} - \frac{1}{\beta} H(\hat{\pi}_t) \right] - \frac{T}{\beta} \log(p^\beta_*), \\
\C{SPR}_T^\beta
&= \B{E} \left[ \C{SR}_t^{\op{MAB}} - \frac{1}{\beta} H(\hat{\pi}_t) \right] - \frac{1}{\beta} \log(p^\beta_*).
\end{align*}
In particular, as $\beta \to \infty$, both $\B{E} \left[ \C{AR}_T^\beta \right]$ and $\B{E} \left[ \C{PR}_T^\beta \right]$ tend to $\B{E} \left[ \C{R}_T^{\op{MAB}} \right]$ and $\B{E} \left[ \C{SPR}_T^\beta \right]$ tends to $\B{E} \left[ \C{SR}_t^{\op{MAB}} \right]$.
\end{theorem}
See Appendix~\ref{apx:MAS_to_MAB} for the proof.

Note that, for any fixed problem and $\beta$, the value of $p_*^\beta$ is constant.
Therefore, this theorem demonstrates that the notion of policy-level $\beta$-regret defined here is equivalent to entropy-regularized bandits, in the sense that is commonly considered in the literature.
Also note that, in the case of $\beta=1$, we have $\C{AR}_T^\beta = T . \C{SAR}_T^{\op{r-KL}}$ and $\C{SPR}_T^\beta = \C{SPR}_T^{\op{r-KL}}$.
Thus the above theorem also demonstrates that, as $\beta \to \infty$, these notions of regret in sampling converge and the limit is the multi-armed bandit regret.
It is not surprising that, for any finite $\beta$, minimizing $\beta$-regret is not fundamentally different from minimizing $1$-regret.
As we will discuss in Appendix~\ref{apx:alg-beta}, for Algorithm~\ref{alg:AS}, replacing the step $\hat{p}_{t+1} \gets \op{softmax}(\hat{r}_t)$ with $\hat{p}_{t+1} \gets \op{softmax}(\beta \hat{r}_t)$ results in an algorithm that achieves low $\beta$-regret.
We refer to Appendix~\ref{apx:alg-beta} for more details.

\section*{Acknowledgments}
This research is in part supported by CIFAR AI Chairs program, NSERC Discovery and IVADO IAR3.

\clearpage
\bibliographystyle{apalike}
\bibliography{references}

\clearpage
\onecolumn
\appendix

\section{RELATED WORKS}\label{sec:related}

We discussed some of the related works, specifically those related to the fine-tuning of pretrained models in the introduction.
In the following, we discuss other related works.

\subsection{Multi-armed bandits}\label{sec:related:MAB}
Originating from the work of \citep{robbins52_some}, the multi-armed bandit problem captures the fundamental trade-off between exploration—gathering information about the potential rewards of each arm—and exploitation—leveraging the current knowledge to maximize immediate rewards.
Over the decades, numerous algorithms have been developed to effectively address this problem.
Notable among these are the \emph{Upper Confidence Bound (UCB)} algorithms introduced by \citep{Auer2003UsingCB}, which select arms based on optimistic estimates of their potential rewards, and \emph{Thompson Sampling} \citep{thompson33,pmlr-v23-agrawal12}, a Bayesian approach that balances exploration and exploitation by sampling from the posterior distributions of each arm's rewards.
Another classical algorithm is \emph{Explore-Then-Commit (ETC)} algorithm, which is based on ‘certainty equivalence with forcing’ algorithm of~\citep{robbins52_some}.
ETC initially explores all arms for a predetermined number of rounds before committing to the arm with the highest estimated reward, and it has been widely studied for its simplicity and effectiveness~\citep{lattimore2020bandit}.
The distinct separation of exploration and exploitation makes ETC an excellent candidate for us to adapt to the multi-armed sampling problem, as it enables the study of the exploration-exploitation trade-off in one of the most simple scenarios.
Note that ETC is a suboptimal algorithm that requires $\Omega(T^{2/3})$ exploration to guarantee regret of $O(T^{2/3})$ while the optimal minmax regret bound is $O(T^{1/2})$.

\subsection{Markov Chain Monte Carlo}\label{sec:related:MCMC}

\emph{Markov Chain Monte Carlo (MCMC)} methods are a class of algorithms widely employed for sampling from complex probability distributions
\citep{metropolis53_equat_state_calcul_fast_comput_machin,hastings70_monte_carlo_markov,geman84_stoch_relax_gibbs_distr_bayes_restor_images}.
MCMC approaches construct a Markov chain whose stationary distribution matches the target distribution, enabling the generation of representative samples through iterative stochastic processes. 

Convergence of Markov chains is typically measured based on their mixing time.
The fundamental theorem of Markov chains states that, under suitable conditions, the total variation distance of the chain distribution to the target distribution converges exponentially; see~\citet{levin2017markov}.
As discussed in Section~\ref{sec:action-vs-policy-regret}, the chain distribution corresponds to $\B{E}_{\omega^\pi}[\hat{\pi}_t]$ in our notation.
Therefore the fundamental theorem of Markov provides bounds for $\C{SPR}^{\op{env}, \op{TV}}$.
At first glance, this seems to contradict the lower bound in Theorem~\ref{thm:lower-bound}.
However, Theorem~\ref{thm:lower-bound} can only be applied when the environment is stochastic.
In a deterministic case, typically assumed by MCMC methods, we only need to sample each arm once to find the exact target distribution. 
Thus, after sampling each of the $k$ arms once, we may simply sample according to the exact target distribution which results in $\C{SPR}^{\op{env}, d}_t = 0$, for any statistical distance $d$ and all $t \geq k+1$.

\subsection{Adaptive importance sampling and DAISEE}\label{sec:related:AIS}

Adaptive Importance Sampling is a Monte Carlo technique in which proposal distributions used for sampling are iteratively updated based on the result of the evaluation of previously sampled points \cite{cappe2008adaptive}.

Closely related to our work is~\citep{lu2018exploration} which frames AIS as an online learning algorithm.
Their formulation could be seen as a form of multi-armed sampling problem with cumulative policy-level forward-KL regret.
They propose an algorithm, namely DAISEE, which uses the \emph{optimism principle} as an explicit exploration mechanism, similar to that of UCB in multi-armed bandit problems.
They obtain a regret bound of $\tilde{O}(T^{1/2})$.
In contrast, for the same notion of regret, we obtain a regret bound of $\tilde{O}(1)$ with minor exploration and $\hat{O}^{1/2}(1)$ with no exploration.
It is important to note that the exploration in DAISEE is done by adding the confidence bound to unnormalized probabilities.
This is in contrast with our algorithm and also KL-UCB~\citep{zhao25_logar_regret_onlin_kl_regul_reinf_learn} where the the optimism term is added to unnormalized \emph{log}-probabilities.
\cref{fig:main} compares GAS, GASW with logarithmic exploration, and KL-UCB  with DAISEE, where other algorithms seem to outperform DAISEE.

\subsection{Neural Samplers}
Neural networks are increasingly used in the context of sampling through diverse means. Some are inspired by ideas from deep generative models, such as diffusion \citep[e.g.,][]{zhang2021path,vargas2023denoising} and normalizing flows \citep[e.g.,][]{noe2019boltzmann,midgley2022flow}, and some 
are inspired by ideas from RL, including generative flow networks \citep{bengio2021flow,bengio2023gflownet} and entropy-regularized RL \citep{fox2015taming,haarnoja2017reinforcement} (and its applications in fine-tuning of pretrained models as we have discussed in the introduction); see also \citep{tiapkin2024generative}. The fact that exploration does not seem to become an issue in these sampling counterparts to RL is consistent with our findings. An exception is  
 \cite{rector2023thompson} which uses Thompson sampling to encourage exploration for sampling in generative flow networks.

\section{THE NECESSITY OF NOISE}\label{apx:noise}

Given that we are working in the finite-arm setting, it is necessary to assume that observations are noisy. 
Indeed, in both MAS and MAB, if observations were noise-free, we could simply try each arm once to determine its true reward, eliminating any exploration–exploitation tradeoff.

An alternative way to interpret noise is through the lens of approximating a continuous (or otherwise infinite) arm space with a finite one. 
Specifically, we can partition the arms of the original continuous problem into a finite number of super-arms and define a new problem over these aggregated arms. 
To complete this formulation, each super-arm must be assigned a reward.
If we denote a super-arm by $\mathbf{x}$, corresponding to a subset of arms $a \in A_{\mathbf{x}}$, a natural approach is to define a probability distribution $\rho_{\mathbf{x}}$ over $A_{\mathbf{x}}$ and set $r(\mathbf{x}) = r(a)$, where $a$ is drawn according to $\rho_{\mathbf{x}}$.

This approach is commonly adopted in the reinforcement learning and sampling literature. For instance, many algorithms for Lipschitz bandit problems rely on this idea, and \citet{lu2018exploration} apply it to extend their MAS algorithm (DAISEE) to continuous settings. Consequently, even if the reward function in the original continuous problem is deterministic, it remains reasonable to model noise in the corresponding finite-arm version. Introducing stochasticity in this way helps preserve the essential exploration–exploitation dynamics of the original problem, ensuring that analysis in the finite setting provides meaningful insights into exploration–exploitation tradeoffs in more general, continuous environments.

\section{THE ROLE OF SOFTMAX}\label{apx:softmax}

As discussed earlier, it is important in our setting that observations are noisy.
Consider first the case where we observe noisy samples of the unnormalized probabilities of the target distribution. Since unnormalized probabilities are always non-negative, additive noise is not necessarily a meaningful modeling choice. A more appropriate alternative is to introduce \emph{multiplicative} noise. Taking the logarithm of these noisy values then yields \emph{unnormalized log-probabilities} with \emph{additive} noise.

Recall that our goal is to estimate or match the target distribution. When working with unnormalized log-probabilities, i.e., $r_i = \log(p_i) + C$ for some constant $C \in \B{R}$, applying the softmax function recovers the original distribution $p$. Equivalently, if we had instead modeled unnormalized probabilities with multiplicative noise, there would be no need to include the softmax transformation explicitly.

As noted in \cref{sec:related:AIS}, DAISEE performs exploration by adding a confidence bound directly to unnormalized probabilities. This contrasts with our algorithm, as well as KL-UCB, where the optimism term is added in the space of unnormalized \emph{log}-probabilities. Both theoretical regret bounds and empirical results from toy experiments suggest that applying additive noise in the space of unnormalized probabilities is suboptimal.

Given that additive noise in the log-probability space is both mathematically convenient and more consistent with the structure of multi-armed bandit problems, we adopt the formulation where the observation at each time step corresponds to unnormalized log-probabilities with additive noise.

\section{USEFUL LEMMAS}

\begin{lemma}\label{lem:bound-on-softmax}
Let $r = (r_i)_{i = 1}^k$ and $r' = (r'_i)_{i = 1}^k$ be two vectors in $\B{R}^d$ and let $p = \op{softmax}(r)$ and $p' = \op{softmax}(r')$.
Also let $\epsilon := \max_{1 \leq i \leq k} | r_i - r'_i |$.
Then we have
\begin{align*}
p_i \exp(-2 \epsilon) \leq p'_i \leq p_i \exp(2 \epsilon)
\quad\t{ and }\quad
| p_i - p'_i | \leq 2 \epsilon,
\end{align*}
for all $1 \leq i \leq k$.
As a consequence, we have
$d^{\op{TV}}(p, p') \leq 2 \epsilon$.
\end{lemma}
\begin{proof}
To prove the first claim, we note that
\begin{align*}
\log\left( \frac{p'_i}{p_i} \right)
= \log\left( \frac{\op{softmax}(r')_i}{\op{softmax}(r)_i} \right)
= \log\left( \frac{\exp(r'_i)}{\exp(r_i)} \right)
+ \log\left( \frac{\sum_j \exp(r_j)}{\sum_j \exp(r'_j)} \right) \\
\leq \log\left( \frac{\exp(r_i + \epsilon)}{\exp(r_i)} \right)
+ \log\left( \frac{\sum_j \exp(r'_j + \epsilon)}{\sum_j \exp(r'_j)} \right) 
= \epsilon + \epsilon
= 2 \epsilon.
\end{align*}
The inequality in the other direction is proven similarly.

Using the above bounds, we see that
\begin{align*}
p'_i - p_i 
\leq p'_i - p'_i \exp(-2 \epsilon)
= p'_i (1 - \exp(-2 \epsilon))
\leq 2 \epsilon p'_i
\leq 2 \epsilon,
\end{align*}
where the last inequality follows from the fact that $1 - \exp(-x) \leq x$ for all $x \in \B{R}$.
Similarly we see that $p'_i - p_i \geq -2 \epsilon$.

To complete the proof of the lemma, we note that
\begin{align*}
d^{\op{TV}}(p, p')
&= \sum_{i = 1}^k \max\{ 0, p'_i - p_i \}
\leq \sum_{i = 1}^k 2 \epsilon p'_i
= 2 \epsilon.
\qedhere
\end{align*}
\end{proof}

\begin{lemma}\label{lem:d-prime}
Let $d$ be a statistical distance and define $d'(p, p') := \inf_q d(q, p) + d(q, p')$.
Then the following are true.
\begin{enumerate}
\item[(i)] If $d$ is a metric, e.g. $d^{\op{TV}}$, then $d' = d$.
\item[(ii)] If $d(p, p') = D_{\op{KL}}(p || p')$ or $d(p, p') = D_{\op{KL}}(p' || p)$, then we have 
$d'(p, p') \geq d^{\op{TV}}(p, p')^2$.
\end{enumerate}
\end{lemma}

\begin{proof}
\textbf{(i)} 
We note that
\begin{align*}
d'(p, p') = \inf_q d(q, p) + d(q, p') \leq d(p, p) + d(p, p') = d(p, p').
\end{align*}
On the other hand, using the triangle inequality, we have
\begin{align*}
d'(p, p') = \inf_q d(q, p) + d(q, p') \geq \inf_q d(p, p') = d(p, p').
\end{align*}.

\textbf{(ii)}
Assume $d(p, p') = D_{\op{KL}}(p || p')$.
Using Pinsker's inequality, we have
\begin{align*}
d(q, p) + d(q, p')
&\geq 2\left( d^{\op{TV}}(q, p)^2 + d^{\op{TV}}(q, p')^2 \right) \\
&\geq \left( d^{\op{TV}}(q, p) + d^{\op{TV}}(q, p') \right)^2
\geq d^{\op{TV}}(p, p')^2.
\end{align*}
The same argument applies to the case $d(p, p') = D_{\op{KL}}(p' || p)$.
\end{proof}

\section{PROOF OF THEOREM~\ref{thm:KL-vs-TV}}\label{apx:thm:KL-vs-TV}

\begin{proof}
We write the proof for $d^{\op{r-KL}}$. 
The proof for $d^{\op{f-KL}}$ is similar.
Using Pinsker's inequality, we have
\begin{align*}
\C{R}^{f, d^{\op{r-KL}}}
= D_{\op{KL}}(f \| p) 
&\geq 2 d^{\op{TV}}(f, p)^2 
= 2 \left( \C{R}^{f, d^{\op{TV}}} \right)^2.
\end{align*}
To prove the other side of the inequality for $d^{\op{r-KL}}$, we use inverse Pinsker's inequality\footnote{See Remark 1 in~\cite{sason15_upper_rényi}}
\begin{align*}
\C{R}^{f, d^{\op{r-KL}}}
= D_{\op{KL}}(f \| p)
&\leq \frac{2}{\alpha} d^{\op{TV}}(f, p)^2 
= \frac{2}{\alpha} \left( \C{R}^{f, d^{\op{TV}}} \right)^2.
\end{align*}
Finally, to prove the inequality with expected values, we note that
\begin{align*}
&\B{E}\left[ \C{R}^{f, d^{\op{r-KL}}} \right]
= \B{E}\left[ D_{\op{KL}}(f \| p) \right]
\geq \B{E}\left[ 2 d^{\op{TV}}(f, p)^2 \right] \\
&\qquad\geq 2 \B{E}\left[ d^{\op{TV}}(f, p) \right]^2
= 2 \B{E}\left[ \C{R}^{f, d^{\op{TV}}} \right]^2.
\qedhere
\end{align*}
\end{proof}

\section{PROOF OF LEMMA~\ref{lem:lower-bound}}\label{apx:lem:lower-bound}

\begin{proof}
For a multi-armed sampling environment $\nu$, 
let $A_{\nu}$ be the $T \times k$ matrix where the $(t, i)$ element is $r_i$.
We define a \emph{noise matrix} for $\nu$ as a random matrix $\varepsilon = (\varepsilon_{t, i})_{1 \leq t \leq T, 1 \leq i \leq k} = \omega^\nu - A_\nu$. Given a policy $\pi$ and a noise matrix $\varepsilon$, we may run the policy where the observation $\BF{x}_t$ at time-step $t$ is given by $r_{\BF{a}_t} + \varepsilon_{t, \BF{a}_t}$.
Hence a noise matrix $\varepsilon$ captures all the randomness of the environment and we may use $\varepsilon$ and $\omega^\nu$ interchangeably.
Note for any given realization of the noise matrix, $\BF{a}_t$ is still a random variable since the policy $\pi$ is not necessarily deterministic.

Let $\BF{x}_t$ and $\BF{x}'_t$ denote the sequence of observations when running the policy $\pi$ against $\nu$ and $\nu'$ with noise matrices $\varepsilon$ and $\varepsilon'$.
Also let $A_{\nu, \nu'} := A_\nu - A_{\nu'}$ and $\varepsilon' = A_{\nu, \nu'} + \varepsilon$.
The key property of this matrix is that it if $\varepsilon' = A_{\nu, \nu'} + \varepsilon$, then we have 
\begin{align*}
  \BF{x}'_t
  = r'_{\BF{a}_t} + \varepsilon'_{t, \BF{a}_t}
  = r'_{\BF{a}_t} + [ A_{\nu, \nu'} ]_{t, \BF{a}_t} + \varepsilon_{t, \BF{a}_t}
  = r_{\BF{a}_t} + \varepsilon_{t, \BF{a}_t}
  = \BF{x}_t.
\end{align*}
In other words, if $\varepsilon' = A_{\nu, \nu'} + \varepsilon$, then the sequence of actions and observations is identical between $(\nu, \pi, \varepsilon)$ and $(\nu', \pi, \varepsilon')$, and therefore the value of the performance function $f$ is the same between the two cases.

Let $\Phi$ denote the probability density function of $N(0, I_{Tk})$, the $Tk$-dimensional unit normal distribution and let 
\begin{align*}
g(\pi, \nu, \varepsilon) := \B{E}_{\omega^\pi} \left[ d(f(\pi, \omega^\pi, \varepsilon + A_\nu), p) \right].
\end{align*}
Hence we have
\begin{align*}
\C{R}^{f, d}(\pi, \nu)
= \B{E}_{\varepsilon} [ g(\pi, \nu, \varepsilon) ] 
= \int g(\pi, \nu, \varepsilon) \Phi(\varepsilon) d \varepsilon.
\end{align*}

If $\varepsilon' = A_{\nu, \nu'} + \varepsilon$, then using the fact that $|r_i - r'_i| \leq \epsilon$, we have
\begin{align*}
\frac{\Phi(\varepsilon)}{\Phi(\varepsilon')}
&= \prod_{i = 1}^k \prod_{t = 1}^T \frac{\Phi(\varepsilon_{t, i})}{\Phi(\varepsilon'_{t, i})} \\
&= \prod_{i = 1}^k \prod_{t = 1}^T \exp\left( - \frac{1}{2} \varepsilon_{t, i}^2 + \frac{1}{2} (\varepsilon'_{t, i})^2 \right) \\
&= \prod_{i = 1}^k \prod_{t = 1}^T \exp\left( - \frac{1}{2}(\varepsilon'_{t, i} + r'_i - r_i)^2 + \frac{1}{2}(\varepsilon'_{t, i})^2 \right) \\
&= \exp \left( - \frac{1}{2} T \sum_{i = 1}^T (r'_i - r_i)^2 - \sum_{i = 1}^k \sum_{t = 1}^T \varepsilon'_{t, i} (r'_i - r_i) \right) \\
&\geq \exp \left(- \frac{1}{2} T k \epsilon^2 - \epsilon \sum_{i = 1}^k \sum_{t = 1}^T \varepsilon'_{t, i} \right).
\end{align*}
Let 
$E' := \left\{ \varepsilon' \mid | \sum_{i = 1}^k \sum_{t = 1}^T \varepsilon'_t | < \sqrt{2 T k} \right\}$.
It follows that, for any $\varepsilon \in E'$, we have 
\begin{align*}
\frac{\Phi(\varepsilon)}{\Phi(\varepsilon')}
&\geq \exp \left(- \frac{1}{2} T k \epsilon^2 - \epsilon \sum_{i = 1}^k \sum_{t = 1}^T \varepsilon'_{t, i} \right)
\geq \exp \left( - \frac{1}{2} T k \epsilon^2 - \epsilon \sqrt{2 T k} \right).
\end{align*}
Using the definition of $d'$, we have
\begin{align*}
d(f , p) + d(f , p') \geq d'(p, p'),
\end{align*}
where we use $f$ to denote the output of the function $f$.
Therefore, if $\varepsilon' = \varepsilon + A_{\nu, \nu'} \in E'$, then we have
\begin{align*}
d(f , p) \Phi(\varepsilon)
&\geq \max\left\{ 0, d'(p, p') - d(f , p') \right\} \Phi(\varepsilon) \\
&= \max\left\{ 0, d'(p, p') - d(f , p') \right\} \Phi(\varepsilon') \frac{\Phi(\varepsilon)}{\Phi(\varepsilon')} \\
&\geq \max\left\{ 0, d'(p, p') - d(f , p') \right\} \Phi(\varepsilon') \exp \left( - \frac{1}{2} T k \epsilon^2 - \epsilon \sqrt{2 T k} \right) \\
&\geq \left( d'(p, p') - d(f , p') \right) \Phi(\varepsilon') \exp \left( - \frac{1}{2} T k \epsilon^2 - \epsilon \sqrt{2 T k} \right),
\end{align*}
which implies that
\begin{align*}
g(\pi, \nu, \varepsilon) \Phi(\varepsilon)
&\geq \B{E}_{\omega^\pi} \left[ \left( d'(p, p') - d(f , p') \right) \Phi(\varepsilon') \exp \left( - \frac{1}{2} T k \epsilon^2 - \epsilon \sqrt{2 T k} \right) \right] \\
&= \exp \left( - \frac{1}{2} T k \epsilon^2 - \epsilon \sqrt{2 T k} \right) \B{E}_{\omega^\pi} \left[ d'(p, p') - d(f , p') \right] \Phi(\varepsilon') \\
&= \exp \left( - \frac{1}{2} T k \epsilon^2 - \epsilon \sqrt{2 T k} \right) \left( d'(p, p') - g(\pi, \nu', \varepsilon') \right) \Phi(\varepsilon') \\
\end{align*}
Therefore
\begin{align*}
\int g(\pi, \nu, \varepsilon) \Phi(\varepsilon) d \varepsilon
&\geq \int_{E' - A_{\nu, \nu'}} g(\pi, \nu, \varepsilon) \Phi(\varepsilon) d \varepsilon 
= \int_{E'} g(\pi, \nu, \varepsilon) \Phi(\varepsilon) d \varepsilon' \\
&\geq \exp \left( - \frac{1}{2} T k \epsilon^2 - \epsilon \sqrt{2 T k} \right) \int_{E'} \left( d'(p, p') - g(\pi, \nu', \varepsilon') \right) \Phi(\varepsilon') d \varepsilon' \\
&= \exp \left( - \frac{1}{2} T k \epsilon^2 - \epsilon \sqrt{2 T k} \right) \left( d'(p, p') \B{P}(E') - \int_{E'} g(\pi, \nu', \varepsilon') \Phi(\varepsilon') d \varepsilon' \right) \\
&\geq \exp \left( - \frac{1}{2} T k \epsilon^2 - \epsilon \sqrt{2 T k} \right) \left( d'(p, p') \B{P}(E') - \int g(\pi, \nu', \varepsilon') \Phi(\varepsilon') d \varepsilon' \right).
\end{align*}
According to Hoeffding's inequality, we have $\B{P}(\varepsilon' \in E') \geq 1 - 1/e > 1/2$.
Therefore we have 
\begin{align*}
\int g(\pi, \nu, \varepsilon) \Phi(\varepsilon) d \varepsilon
&\geq \exp \left( - \frac{1}{2} T k \epsilon^2 - \epsilon \sqrt{2 T k} \right) \left( d'(p, p') \B{P}(E') - \int g(\pi, \nu', \varepsilon') \Phi(\varepsilon') d \varepsilon' \right) \\
&\geq \frac{1}{2} \exp \left( - \frac{1}{2} T k \epsilon^2 - \epsilon \sqrt{2 T k} \right) d'(p, p')
- \int g(\pi, \nu', \varepsilon') \Phi(\varepsilon') d \varepsilon',
\end{align*}
which completes the proof.
\end{proof}

\section{PROOF OF THEOREM~\ref{thm:lower-bound}} \label{apx:thm:lower-bound}

\begin{proof}
If $\epsilon = \frac{1}{\sqrt{T k}}$, then we have
\begin{align*}
\frac{1}{2} \exp \left( - \frac{1}{2} T k \epsilon^2 - \epsilon \sqrt{2 T k} \right)
= \frac{1}{2} \exp(- \frac{1}{2} - \sqrt{2}) 
> \frac{1}{14},
\end{align*}
and
\begin{align*}
\exp(2 \epsilon)
= \exp \left( \frac{2}{\sqrt{T k}} \right)
\leq \exp \left( \frac{2}{\sqrt{5 \cdot 2}} \right)
< \exp \left( \frac{2}{3} \right)
< 2.
\end{align*}
We want to choose $r$ and $r'$ in Lemma~\ref{lem:lower-bound} such that $| r_i - r'_i | \leq \epsilon$.
Let
\begin{align*}
\forall 1 \leq i \leq k
,\quad
r_i = \epsilon
,\quad
r'_i = \begin{cases}
2\epsilon   \quad&\t{if } i < (k+1)/2; \\
\log \left( \frac{\exp(2\epsilon) + 1}{2} \right)   \quad&\t{if } k \t{ is odd and } i = (k+1)/2; \\
0  \quad&\t{if } i > (k+1)/2.
\end{cases}
\end{align*}
It follows that
\begin{align*}
\forall 1 \leq i \leq k
,\quad
p_i = \frac{1}{k}
,\quad
p'_i = \begin{cases}
\frac{2 \exp(2\epsilon)}{k (\exp(2\epsilon) + 1)}   \quad&\t{if } i < (k+1)/2; \\
\frac{1}{k}                                         \quad&\t{if } k \t{ is odd and } i = (k+1)/2; \\
\frac{2}{k (\exp(2\epsilon) + 1)}                   \quad&\t{if } i > (k+1)/2. \\
\end{cases}
\end{align*}
Hence we have
\begin{align*}
d(p', p)
&= \frac{1}{2} \sum_{i = 1}^k | p'_i - p_i | 
= \begin{cases}
\frac{\exp(2 \epsilon) - 1}{2 (\exp(2 \epsilon) + 1)}    &\t{if } k \t{ is even} \\
\frac{k-1}{k} \cdot \frac{\exp(2 \epsilon) - 1}{2 (\exp(2 \epsilon) + 1)}    &\t{if } k \t{ is odd}
\end{cases} \\
&\geq \frac{2}{3} \cdot \frac{\exp(2 \epsilon) - 1}{2 (\exp(2 \epsilon) + 1)} 
\geq \frac{2}{3} \cdot \frac{2 \epsilon}{2 (\exp(2 \epsilon) + 1)} 
\geq \frac{2 \epsilon}{9}
= \frac{2}{9 \sqrt{T k}},
\end{align*}
where the second inequality follows from the fact that $\exp(x) \geq 1 + x$ for all $x \in \B{R}$ and the last inequality follows from the fact that $\exp(2\epsilon) < 2$.
Now we may use Lemma~\ref{lem:lower-bound} to see that
\begin{align*}
\sup_{\nu} \B{E}\left[ \C{R}^{f, d^{\op{TV}}}(\pi, \nu) \right]
&\geq \frac{1}{2} \left(
  \B{E}\left[ \C{R}^{f, d^{\op{TV}}}(\pi, \nu) \right] 
  + \B{E}\left[ \C{R}^{f, d^{\op{TV}}}(\pi, \nu') \right] \right) \\
&\geq \frac{1}{28} d^{\op{TV}}(p', p)
\geq \frac{1}{126 \sqrt{T k}}.
\end{align*}
For $d^{\op{r-KL}}$, we have
\begin{align*}
\sup_{\nu} \B{E}\left[ \C{R}^{f, d^{\op{r-KL}}}(\pi, \nu) \right]
&\geq \frac{1}{2} \left(
  \B{E}\left[ \C{R}^{f, d^{\op{r-KL}}}(\pi, \nu) \right] 
  + \B{E}\left[ \C{R}^{f, d^{\op{r-KL}}}(\pi, \nu') \right] \right) 
\geq \frac{1}{28} d'(p', p) \\
&\geq \frac{1}{28} d(p', p)^2 
\geq \frac{1}{28} \left( \frac{2}{9 \sqrt{T k}} \right)^2
= \frac{1}{567 T k}.
\end{align*}
The proof for $d^{\op{f-KL}}$ is identical.
\end{proof}

\section{PROOF OF THEOREMS~\ref{thm:AS} and~\ref{thm:AS2}}\label{apx:thm:AS}

\begin{algorithm2e} \SetKwInOut{Input}{Input}\DontPrintSemicolon
\caption{Active Sampling with Warm-up within Confidence Bounds (ASWCB)}
\label{alg:AS2}
\small
\Input{Horizon $T$, number of arms $k$, exploration factor $1 \leq M \leq T/k$, confidence factor $C \geq 0$}
\For{$t = 1$ to $T$}{
  \eIf{$t \leq Mk$}{
    Play $\BF{a}_t \gets (t \mod k) + 1$ \;
  }{
    Sample $1 \leq \BF{a}_t \leq k$ according to $\hat{\BF{p}}_t$ and play $\BF{a}_t$ \;
  }
  \For{$1 \leq i \leq k$}{
    Choose $\hat{\BF{r}}_{t, i}$ such that
    $\left| \hat{\BF{r}}_{t, i} - \frac{1}{\BF{n}_{t, i}} \sum_{t'=1}^t \BF{x}_{t'} \BF{1}_{\BF{a}_{t'} = i} \right| \leq 2 C \sigma \sqrt{\frac{\log(T) }{\BF{n}_{t, i}}}$ \;
  }
  $\hat{\BF{p}}_{t+1} \gets \op{softmax}(\hat{\BF{r}}_t)$ \;
}
\end{algorithm2e}

In this section, we will provide the regret bounds and proofs for Theorems~\ref{thm:AS} and~\ref{thm:AS2}.
Algorithm~\ref{alg:AS2} provides a pseudo-code for the family of algorithms described in Theorem~\ref{thm:AS2}.
In particular, we can immediately see that setting $C = 0$, reduces Algorithm~\ref{alg:AS2} to Algorithm~\ref{alg:AS}.

We start with some notations and definitions.
Let $\delta = T^{-2}$ and $\sigma_* = (1 + C) \sigma$.
Also let 
\begin{align*}
\epsilon &:= \sqrt{\frac{2 (1 + C) \sigma^2 \log(1/\delta)}{M}} = \sqrt{\frac{2 \sigma_*^2 \log(1/\delta)}{M}}, \\
T_0 &:= \max\left\{ 
    \min\left\{ M k, \frac{k}{1 - \exp(-2 \epsilon)} \right\}, 
    \frac{8}{\alpha^2} \exp(4 \epsilon) \log(1/\delta)
\right\},
\end{align*}
where $\alpha = \min_i p_i$.
Note that if $M \geq 18 \sigma_*^2 \log(1/\delta) = 36 \sigma_*^2 \log(T)$, then
\begin{align*}
\epsilon 
= \sqrt{\frac{2 \sigma_*^2 \log(1/\delta)}{M}}
\leq \sqrt{\frac{2 \sigma_*^2 \log(1/\delta)}{18 \sigma_*^2 \log(1/\delta)}}
= \frac{1}{3},
\end{align*}
which implies that $\exp(\epsilon) \leq \frac{3}{2}$.
More generally, if $M = \Omega(\log(T))$, then $\epsilon = O(1)$ and $T_0 = O(\log(T))$.
On the other hand, if $M = 1$, then $\epsilon = O((\log T)^{1/2})$ and therefore
$\exp(A \epsilon) = \hat{O}^{1/2}(1)$ for any $A > 0$.
Thus we also have $T_0 = O(\exp(4 \epsilon) \log(T)) = \tilde{O}(\exp(4 \epsilon)) = \hat{O}^{1/2}(1)$.

Recall that $\BF{n}_{t, i} = \sum_{t' = 1}^t \BF{1}_{\BF{a}_{t'} = i}$.
We define $\BF{n}'_{t, i} = \BF{n}_{t, i} - \sum_{t'=1}^{t} \hat{\pi}_{t', i}$.
We also define $\C{E}'_t$ to be the event that the inequality
\begin{align*}
| \BF{n}'_{t, i} - \BF{n}'_{0, i} |
\leq \sqrt{2 t \log(1/\delta)}
\end{align*}
holds for all $1 \leq i \leq k$.
We have
\begin{align*}
\B{E}[\BF{n}'_{t+1, i} | \BF{n}'_{1, i}, \cdots, \BF{n}'_{t, i} ] 
&= \B{E}[\BF{1}_{\BF{a}_{t+1} = i} - \hat{\pi}_{t+1, i} + \BF{n}'_{t, i} | \BF{n}'_{1, i}, \cdots, \BF{n}'_{t, i} ] \\
&= \B{E}[\BF{1}_{\BF{a}_{t+1} = i} - \hat{\pi}_{t+1, i} | \BF{n}'_{1, i}, \cdots, \BF{n}'_{t, i} ] + \BF{n}'_{t, i} 
= \BF{n}'_{t, i}.
\end{align*}
Therefore, $\{\BF{n}'_{t, i}\}_{t = 0}^{T}$ is a martingale where 
\begin{align*}
| \BF{n}'_{t+1, i} - \BF{n}'_{t, i} | = | \BF{1}_{\BF{a}_{t+1} = i} - \hat{\pi}_{t+1, i} | \leq 1.
\end{align*}
Hence we may use Azuma-Hoeffding's inequality to see that, for each $1 \leq i \leq k$, the inequality 
\begin{align}\label{eq:n_prime_azuma}
\BF{n}'_{t, i} 
= | \BF{n}'_{t, i} - \BF{n}'_{0, i} |
\leq \sqrt{2 t \log(1/\delta)}
\end{align}
occurs with probability $1 - \delta$.
Hence $\B{P}(\C{E}'_t) \geq 1 - k\delta$.
On the other hand, according to Hoeffding's inequality, with probability $1 - \delta$, we have
\begin{align*}
\left| \frac{1}{\BF{n}_{t, i}} \sum_{t'=1}^t \BF{x}_{t'} \BF{1}_{\BF{a}_{t'} = i} - r_i \right|
\leq \sqrt{\frac{2 \sigma^2 \log(1/\delta)}{\BF{n}_{t, i}}}
\end{align*}
and therefore, since $\delta = T^{-2}$, we see that
\begin{align*}
| \hat{\BF{r}}_{t, i} - r_i |
&\leq \left| \frac{1}{\BF{n}_{t, i}} \sum_{t'=1}^t \BF{x}_{t'} \BF{1}_{\BF{a}_{t'} = i} - r_i \right| 
+ \left| \hat{\BF{r}}_{t, i} - \frac{1}{\BF{n}_{t, i}} \sum_{t'=1}^t \BF{x}_{t'} \BF{1}_{\BF{a}_{t'} = i} \right| \\
&\leq (1+C)\sqrt{\frac{2 \sigma^2 \log(1/\delta)}{\BF{n}_{t, i}}}
= \sqrt{\frac{2 \sigma_*^2 \log(1/\delta)}{\BF{n}_{t, i}}}
\end{align*}
Let $\C{E}_t$ be the event where the above inequality holds for all $1 \leq i \leq k$.
It follows that $\B{P}(\C{E}_t) \geq 1 - k \delta$.
Finally, we define $\C{E}$ to be the event $\bigcap_{t' = Mk}^{T} (\C{E}_{t'} \cap \C{E}'_{t'})$.
Thus $\B{P}(\C{E}) \geq 1 - 2 k T \delta$.

As can be seen in Algorithm~\ref{alg:AS2}, we have $\hat{\pi}_t = \hat{\BF{p}}_t$ for all $t \geq Mk+1$.
On the other hand, for $1 \leq t \leq Mk$, $\hat{\pi}_t$ is the Dirac distribution with mass at $(t \mod k) + 1$.
Therefore, for all $t \geq Mk$, we have
\begin{align*}
\BF{n}_{t, i} 
= \sum_{t'=1}^{t} \hat{\pi}_{t', i} + \BF{n}'_{t, i}
= M + \sum_{t'=Mk+1}^{t} \hat{\BF{p}}_{t', i} + \BF{n}'_{t, i}.
\end{align*}

We provide the proof in a series of lemmas.

\begin{lemma}\label{lem:ASE:base_TV}
If $T_0 \leq t \leq T$, and the event $\C{E}$ occurs, then
\begin{align}
d^{\op{TV}}(\hat{\BF{p}}_t, p) 
\leq 4 \sigma_* \exp(\epsilon) \sqrt{\frac{\log(1/\delta)}{\alpha t}}.
\label{eq:lem:ASE:base_TV}
\end{align}
\end{lemma}

\begin{proof}
For $t' \geq M k$, the event $\C{E}_{t'}$ occurs and therefore we have
\begin{align*}
| \hat{\BF{r}}_{t', i} - r_i |
\leq \sqrt{\frac{2 \sigma_*^2 \log(1/\delta)}{\BF{n}_{t', i}}}
\leq \sqrt{\frac{2 \sigma_*^2 \log(1/\delta)}{M}}
= \epsilon.
\end{align*}
If $t \geq Mk + 1$, using Lemma~\ref{lem:bound-on-softmax} and the fact that $\alpha \leq 1/k$, we have
\begin{align}
\sum_{t' = 1}^t \hat{\pi}_{t', i}
= M + \sum_{t' = Mk+1}^{t} \hat{\BF{p}}_{t', i} 
&\geq M + \exp(- 2 \epsilon) \sum_{t' = Mk+1}^{t} p_i 
= M + \exp(- 2 \epsilon) (t - M k) p_i \nonumber\\
&\geq M + \exp(- 2 \epsilon) (t - M k) \alpha 
\geq \exp(- 2 \epsilon) \alpha t.
\label{eq:lem:ASE:base_TV:1}
\end{align}
On the other hand, if $T_0 \leq t \leq M k$, then by definition of $T_0$ we have $t \geq \frac{k}{1 - \exp(-2 \epsilon)}$ and therefore
\begin{align}
\sum_{t' = 1}^t \hat{\pi}_{t', i}
= \lfloor t/k \rfloor
&\geq t/k - 1
\geq \exp(- 2 \epsilon) t/k
\geq \exp(- 2 \epsilon) \alpha t.
\label{eq:lem:ASE:base_TV:2}
\end{align}
Thus, since $\C{E}'_{t}$ occurs, we see that
\begin{align}
\BF{n}_{t, i} 
&= \sum_{t' = 1}^t \hat{\pi}_{t', i} + \BF{n}'_{t, i}
\geq \exp(- 2 \epsilon)\alpha t - \sqrt{2 t \log(1/\delta)}.
\label{eq:lem:ASE:base_TV:3}
\end{align}
The inequality $t \geq T_0 \geq \frac{8}{\alpha^2} \exp(4 \epsilon) \log(1/\delta)$ implies that
\begin{align*}
\sqrt{2 t \log(1/\delta)}
\leq \sqrt{\frac{2 t^2 \log(1/\delta)}{\frac{8}{\alpha^2} \exp(4 \epsilon) \log(1/\delta)}}
= \frac{1}{2} \exp(- 2 \epsilon) \alpha t.
\end{align*}
Plugging this into Inequality~\ref{eq:lem:ASE:base_TV:3}, we see that
\begin{align*}
\BF{n}_{t, i} 
&\geq \exp(- 2 \epsilon)\alpha t - \frac{1}{2} \exp(- 2 \epsilon) \alpha t
= \frac{1}{2} \exp(- 2 \epsilon) \alpha t.
\end{align*}
Hence, using the fact that $\C{E}_{t}$ holds, we see that
\begin{align*}
| \hat{\BF{r}}_{t, i} - r_i |
&\leq 
\sqrt{\frac{2 \sigma_*^2 \log(1/\delta)}{\BF{n}_{t, i}}}
\leq \sqrt{\frac{4 \exp(2 \epsilon) \sigma_*^2 \log(1/\delta)}{\alpha t}}
= 2 \sigma_* \exp(\epsilon) \sqrt{\frac{\log(1/\delta)}{\alpha t}}.
\end{align*}
Therefore, Lemma~\ref{lem:bound-on-softmax} implies that
\begin{align*}
d^{\op{TV}}(\BF{\hat{p}}_t, p)
&\leq 2 \max_{1 \leq i \leq k} | \hat{\BF{r}}_{t, i} - r_i |
\leq 4 \sigma_* \exp(\epsilon) \sqrt{\frac{\log(1/\delta)}{\alpha t}},
\end{align*}
which completes the proof.
\end{proof}

\begin{lemma}[Policy-level total-variation regret]\label{lem:ASE:policy_TV}
If $T_0 \leq t \leq T$, then we have
\begin{align*}
\B{E} \left[ \C{SPR}_t^{\op{TV}} \right]
&\leq 4 \sigma_* \exp(\epsilon) \sqrt{\frac{\log(1/\delta)}{\alpha t}} + 2 k T \delta, \\
\B{E} \left[ \C{CPR}_T^{\op{TV}} \right]
&\leq 8 \sigma_* \exp(\epsilon) \sqrt{\frac{\log(1/\delta) T}{\alpha}}
  + T_0
  + 2 k T \delta.
\end{align*}
Thus, 
$\B{E} \left[ \C{SPR}_T^{\op{TV}} \right] = \tilde{O}(\exp(\epsilon)T^{-1/2})$ 
and 
$\B{E} \left[ \C{CPR}_T^{\op{TV}} \right] = \tilde{O}(\exp(4 \epsilon)T^{1/2})$.
\end{lemma}

\begin{proof}
Assume $T_0 \leq t \leq T$.
Using Lemma~\ref{lem:ASE:base_TV} we see that
\begin{align*}
\B{E} \left[ \C{SPR}_t^{\op{TV}} | \C{E} \right]
&= \B{E} \left[ d^{\op{TV}}(\hat{\BF{p}}_t, p) | \C{E} \right] 
\leq 4 \sigma_* \exp(\epsilon) \sqrt{\frac{\log(1/\delta)}{\alpha t}}.
\end{align*}
Note that
$\C{SPR}_t^{\op{TV}} = d^{\op{TV}}(\hat{\pi}_t \| p) \leq 1$, for all $1 \leq t \leq T$.
Therefore, for $T_0 \leq t \leq T$, we have
\begin{align*}
\B{E} \left[ \C{SPR}_t^{\op{TV}} \right]
&= \B{P}(\C{E}) \B{E} \left[ \C{SPR}_t^{\op{TV}} | \C{E} \right] 
  + \B{P}(\C{E}^c) \B{E} \left[ \C{SPR}_t^{\op{TV}} | \C{E}^c \right] \\
&\leq \B{E} \left[ \C{SPR}_t^{\op{TV}} | \C{E} \right] + \B{P}(\C{E}^c) \\
&\leq 4 \sigma_* \exp(\epsilon) \sqrt{\frac{\log(1/\delta)}{\alpha t}} + 2 k T \delta.
\end{align*}
To bound cumulative regret, we see that when $\C{E}$ occurs,
\begin{align}
\C{CPR}_T^{\op{TV}}
&= \sum_{t = 1}^T \C{SPR}_t^{\op{TV}} \nonumber\\
&= \sum_{t = 1}^{T_0 - 1} \C{SPR}_t^{\op{TV}}
  + \sum_{t = T_0}^T \B{E} \C{SPR}_t^{\op{TV}} \nonumber\\
&\leq T_0 + \sum_{t = T_0}^T \left( 4 \sigma_* \exp(\epsilon) \sqrt{\frac{\log(1/\delta)}{\alpha t}} \right) \nonumber\\
&\leq T_0 + 8 \sigma_* \exp(\epsilon) \sqrt{\frac{\log(1/\delta) T}{\alpha}},
\label{eq:lem:ASE:policy_TV:1}
\end{align}
where we used the fact that $\sum_{n = 1}^N \frac{1}{\sqrt{n}} < 2 \sqrt{N}$ in the last inequality.
Therefore
\begin{align*}
\B{E} \left[ \C{CPR}_T^{\op{TV}} \right]
&= \B{P}(\C{E}) \B{E} \left[ \C{CPR}_T^{\op{TV}} | \C{E} \right] + \B{P}(\C{E}^c) \B{E} \left[ \C{CPR}_T^{\op{TV}} | \C{E}^c \right] \\
&\leq \B{E} \left[ \C{CPR}_T^{\op{TV}} | \C{E} \right] + \B{P}(\C{E}^c) \\
&\leq 8 \sigma_* \exp(\epsilon) \sqrt{\frac{\log(1/\delta) T}{\alpha}}
  + T_0
  + 2 k T \delta.
\qedhere
\end{align*}
\end{proof}

\begin{lemma}[Action-level total-variation regret]\label{lem:ASE:action_TV}
If $T_0 \leq t \leq T$, then we have
\begin{align}
\label{eq:lem:ASE:action_TV}
\B{E} \left[ \C{SAR}_t^{\op{TV}} \right]
&\leq 8 \sigma_* \exp(\epsilon) \sqrt{\frac{\log(1/\delta)}{\alpha t}}
  + \frac{T_0}{t}
  + k \sqrt{ \frac{\log(1/\delta)}{2 t} }
  + 2 k T \delta, \\
\B{E}\left[ \C{CAR}_T^{\op{TV}} \right]
&\leq 8 \sigma_* \exp(\epsilon) \sqrt{\frac{\log(1/\delta) T}{\alpha}} \log T
  + T_0 (1 + \log T)
  + 2 k \sqrt{T \log T}
  + 2 k T^2 \delta.
\end{align}
Thus,
$\B{E} \left[ \C{SAR}_T^{\op{TV}} \right] = \tilde{O}(\exp(4 \epsilon)T^{-1/2})$
and
$\B{E}\left[ \C{CAR}_T^{\op{TV}} \right]
= \tilde{O}(\exp(4 \epsilon)T^{1/2})$.
\end{lemma}

\begin{proof}
We have 
$\BF{n}_{T, i} 
= \sum_{t=1}^{T} \hat{\pi}_{t, i} + \BF{n'}_{T, i}$,
for $1 \leq i \leq k$,
which implies that
$\BF{q}_T
= \frac{1}{T} \left( \sum_{t=1}^{T} \hat{\pi}_t \right) + \frac{\BF{n'}_T}{T}$.
Assume $d = d^{\op{TV}}$.
Using Theorem~\ref{thm:action-level-bounded-by-policy-level} and the fact that $d^{\op{TV}}$ is a metric, we have
\begin{align*}
\C{SAR}_t^{\op{TV}} - \frac{1}{t} \C{CPR}_t^{\op{TV}}
&\leq d^{\op{TV}}\left( \BF{q}_t, p \right)
  - d^{\op{TV}}\left( \BF{q}_t - \frac{\BF{n'}_t}{t}, p \right) 
\leq d^{\op{TV}}\left( \BF{q}_t, \BF{q}_t - \frac{\BF{n'}_t}{t} \right)
= \frac{1}{2 t} \sum_{i = 1}^k \left| \BF{n'}_t \right|.
\end{align*}
Therefore we may use Inequality~\ref{eq:n_prime_azuma} to see that
\begin{align*}
\B{E}\left[ \C{SAR}_t^{\op{TV}} \mid \C{E} \right]
  - \B{E}\left[ \frac{1}{t} \C{CPR}_t^{\op{TV}} \mid \C{E} \right] 
&\leq \B{E}\left[ \C{SAR}_t^{\op{TV}} \mid \C{E} \right]
  - \B{E}\left[ \frac{1}{t} \C{CPR}_t^{\op{TV}} \mid \C{E} \right] \\
&\leq \B{E}\left[ \frac{1}{2 t} \sum_{i = 1}^k \left| \BF{n'}_t \right| \mid \C{E} \right] \\
&\leq \frac{1}{2 t} \sum_{i = 1}^k \B{E}\left[ \left| \BF{n'}_t \right| \right] \\
&\leq k \sqrt{ \frac{\log(1/\delta)}{2 t} }.
\end{align*}
Thus, using the fact that $d$ is bounded by $1$, we see that
\begin{align*}
\B{E}\left[ \C{SAR}_t^{\op{TV}} \right]
&= \B{P}(\C{E}) \B{E}\left[ \C{SAR}_t^{\op{TV}} \mid \C{E} \right]
  + \B{P}(\C{E}^c) \B{E}\left[ \C{SAR}_t^{\op{TV}} \mid \C{E}^c \right] \\
&\leq \B{E}\left[ \C{SAR}_t^{\op{TV}} \mid \C{E} \right] + \B{P}(\C{E}^c) \\
&\leq \frac{1}{t} \B{E}\left[ \C{CPR}_t^{\op{TV}} \mid \C{E} \right]
  + k \sqrt{ \frac{\log(1/\delta)}{2 t} }
  + 2 k T \delta \\
&\leq 8 \sigma_* \exp(\epsilon) \sqrt{\frac{\log(1/\delta)}{\alpha t}}
  + \frac{T_0}{t}
  + k \sqrt{ \frac{\log(1/\delta)}{2 t} }
  + 2 k T \delta.
\end{align*}
where we use Inequality~\ref{eq:lem:ASE:policy_TV:1} in the last line.

Similarly, to bound the cumulative action-level regret, we note that
\begin{align*}
\B{E}\left[ \C{CAR}_T^{\op{TV}} \right]
&= \B{P}(\C{E}) \B{E}\left[ \C{CAR}_T^{\op{TV}} \mid \C{E} \right]
  + \B{P}(\C{E}^c) \B{E}\left[ \C{CAR}_T^{\op{TV}} \mid \C{E}^c \right] \\
&\leq \B{E}\left[ \C{CAR}_T^{\op{TV}} \mid \C{E} \right] + \B{P}(\C{E}^c) T \\
&= \sum_{t = 1}^T \left( \B{E}\left[ \C{SAR}_T^{\op{TV}} \mid \C{E} \right] \right) 
  + \B{P}(\C{E}^c) T \\
&= \sum_{t = 1}^{T_0 - 1} \left( \B{E}\left[ \C{SAR}_T^{\op{TV}} \mid \C{E} \right] \right) 
  + \sum_{t = T_0}^{T} \left( \B{E}\left[ \C{SAR}_T^{\op{TV}} \mid \C{E} \right] \right) 
  + \B{P}(\C{E}^c) T \\
&\leq T_0 
  + \sum_{t = T_0}^T \left( \frac{1}{t} \B{E}\left[ \C{CPR}_t \mid \C{E} \right]
  + k \sqrt{ \frac{\log(1/\delta)}{2 t} } \right)
  + 2 k T^2 \delta \\
&\leq T_0
  + \log (T) \B{E} \left[ \C{CPR}_T^{\op{TV}} \right]
  + 2 k \sqrt{T \log T}
  + 2 k T^2 \delta \\
&\leq 8 \sigma_* \exp(\epsilon) \sqrt{\frac{\log(1/\delta) T}{\alpha}} \log T
  + T_0 (1 + \log T)
  + 2 k \sqrt{T \log T}
  + 2 k T^2 \delta.
\qedhere
\end{align*}
\end{proof}

\begin{lemma}[Policy-level reverse-KL regret]\label{lem:ASE:policy_rKL}
If $T_0 \leq t \leq T$, then we have
\begin{align*}
\B{E} \left[ \C{SPR}_t^{\op{r-KL}} \right]
&\leq \frac{32 \sigma_*^2 \exp(2 \epsilon) \log(1/\delta)}{\alpha^2 t}
  + \frac{4 k T \delta}{\alpha}, \\
\B{E} \left[ \C{CPR}_T^{\op{r-KL}} \right]
&\leq \frac{32 \sigma_*^2 \exp(2 \epsilon) \log(1/\delta) \log(T)}{\alpha^2}
  + \frac{2 T_0}{\alpha} 
  + \frac{4 k T^2 \delta}{\alpha}.
\end{align*}
Thus, 
$\B{E} \left[ \C{SPR}_T^{\op{r-KL}} \right]
= \tilde{O}(\exp(2 \epsilon)T^{-1})$
and 
$\B{E} \left[ \C{CPR}_T^{\op{r-KL}} \right]
= \tilde{O}(\exp(4 \epsilon))$.
\end{lemma}

\begin{proof}
Using Theorem~\ref{thm:KL-vs-TV} and Lemma~\ref{lem:ASE:base_TV} we see that
\begin{align*}
\B{E} \left[ \C{SPR}_t^{\op{r-KL}} | \C{E} \right]
&= \B{E} \left[ \frac{2}{\alpha} \left( \C{SPR}_t^{\op{TV}} \right)^2 | \C{E} \right] 
= \B{E} \left[ \frac{2}{\alpha} \left( d^{\op{TV}}(\hat{\BF{p}}_t, p) \right)^2 | \C{E} \right] 
\leq \frac{32 \sigma_*^2 \exp(2 \epsilon) \log(1/\delta)}{\alpha^2 t}.
\end{align*}
Note that, using inverse Pinsker's inequality, we have
$\C{SPR}_{t'}^{\op{r-KL}} = D_{\op{KL}}(\hat{\pi}_{t'} \| p) \leq \frac{2}{\alpha}$ 
for all $1 \leq t' \leq T$.
Therefore
\begin{align*}
\B{E} \left[ \C{SPR}_t^{\op{r-KL}} \right]
&= \B{P}(\C{E}) \B{E} \left[ \C{SPR}_t^{\op{r-KL}} | \C{E} \right] 
  + \B{P}(\C{E}^c) \B{E} \left[ \C{SPR}_t^{\op{r-KL}} | \C{E}^c \right] \\
&\leq \B{E} \left[ \C{SPR}_t^{\op{r-KL}} | \C{E} \right] + \frac{2}{\alpha}\B{P}(\C{E}^c) \\
&\leq \frac{32 \sigma_*^2 \exp(2 \epsilon) \log(1/\delta)}{\alpha^2 t}
  + \frac{4 k T \delta}{\alpha}.
\end{align*}
To bound the cumulative action-level regret, we note that
\begin{align*}
\B{E} \left[ \C{CPR}_T^{\op{r-KL}} | \C{E} \right]
&= \sum_{t = 1}^T \B{E} \left[ \C{SPR}_t^{\op{r-KL}} | \C{E} \right] \\
&= \sum_{t = 1}^{T_0 - 1} \B{E} \left[ \C{SDR}_t^{\op{r-KL}} | \C{E} \right] + \sum_{t = T_0}^T \B{E} \left[ \C{SDR}_t^{\op{r-KL}} | \C{E} \right] \\
&\leq \frac{2 T_0}{\alpha}
  + \sum_{t = T_0}^T \left( \frac{32 \sigma_*^2 \exp(2 \epsilon) \log(1/\delta)}{\alpha^2 t} \right) \\
&\leq \frac{2 T_0}{\alpha} 
  + \frac{32 \sigma_*^2 \exp(2 \epsilon) \log(1/\delta) \log(T)}{\alpha^2}.
\end{align*}
Therefore
\begin{align*}
\B{E} \left[ \C{CPR}_T^{\op{r-KL}} \right]
&= \B{P}(\C{E}) \B{E} \left[ \C{CPR}_T^{\op{r-KL}} | \C{E} \right] 
  + \B{P}(\C{E}^c) \B{E} \left[ \C{CPR}_T^{\op{r-KL}} | \C{E}^c \right] \\
&\leq \B{E} \left[ \C{CPR}_T^{\op{r-KL}} | \C{E} \right] 
  + \frac{2 T}{\alpha}\B{P}(\C{E}^c) \\
&\leq \frac{32 \sigma_*^2 \exp(2 \epsilon) \log(1/\delta) \log(T)}{\alpha^2}
  + \frac{2 T_0}{\alpha} 
  + \frac{4 k T^2 \delta}{\alpha}.
\qedhere
\end{align*}
\end{proof}

\begin{lemma}[Action-level reverse-KL regret]\label{lem:ASE:action_rKL}
We have
\begin{align*}
\B{E} \left[ \C{SAR}_T^{\op{r-KL}} \right]
&\leq \frac{32 \sigma_*^2 \exp(2 \epsilon) \log(1/\delta) \log(T)}{\alpha^2 T}
  + \frac{2 T_0}{\alpha T} 
  + \frac{\exp(2 \epsilon) k^2 \log(1/\delta)}{\alpha T} \\
  &\qquad\qquad+ \frac{8 \sqrt{2} k \sigma_* \exp(3 \epsilon) \log(1/\delta)}{ \alpha^{3/2} T}
  + \frac{k \exp(2 \epsilon) \sqrt{2 \log(1/\delta)}}{ \alpha T^{3/2}} T_0
  + \frac{4 k T \delta}{\alpha}.
\end{align*}
Thus, 
$\B{E} \left[ \C{SAR}_T^{\op{r-KL}} \right]
= \tilde{O}(\exp(6 \epsilon)T^{-1})$
and
$\B{E} \left[ \C{CAR}_T^{\op{r-KL}} \right] = \tilde{O}(\exp(6 \epsilon))$.
\end{lemma}

\begin{proof}
We only provide the proof for $\B{E} \left[ \C{SAR}_T^{\op{r-KL}} \right]$ for brevity.
The bound for the cumulative regret may be obtained from the simple regret similar to previous lemmas.

Using Theorem~\ref{thm:action-level-bounded-by-policy-level} with $d = d^{\op{r-KL}}$, we have
\begin{align}
\C{SAR}_T^{\op{r-KL}} - \frac{1}{T} \C{CPR}_T^{\op{r-KL}}
&\leq d^{\op{r-KL}}\left( \BF{q}_T, p \right) 
  - d^{\op{r-KL}}\left( \frac{1}{T} \left( \sum_{t=1}^{T} \hat{\pi}_t \right), p \right) \nonumber\\
&= d^{\op{r-KL}}\left( \BF{q}_T, p \right)
  - d^{\op{r-KL}}\left( \BF{q}_T - \frac{\BF{n'}_T}{T}, p \right) \nonumber\\
&= \sum_{i = 1}^k \BF{q}_{T, i} \log\left( \frac{\BF{q}_{T, i}}{p_i} \right)
  - \sum_{i = 1}^k \left( \BF{q}_{T, i} - \frac{\BF{n'}_{T, i}}{T} \right) \log\left( \frac{\BF{q}_{T, i} - \frac{\BF{n'}_{T, i}}{T}}{p_i} \right) \nonumber\\
&= \sum_{i = 1}^k \BF{q}_{T, i} \log\left( \frac{\BF{q}_{T, i}}{p_i} \right)
  - \sum_{i = 1}^k \BF{q}_{T, i} \log\left( \frac{\BF{q}_{T, i} - \frac{\BF{n'}_{T, i}}{T}}{p_i} \right) 
  + \sum_{i = 1}^k \frac{\BF{n'}_{T, i}}{T} \log\left( \frac{\BF{q}_{T, i} - \frac{\BF{n'}_T}{T}}{p_i} \right) \nonumber\\
&= \sum_{i = 1}^k \BF{q}_{T, i} \log\left( \frac{\BF{q}_{T, i}}{ \BF{q}_{T, i} - \frac{\BF{n'}_{T, i}}{T} } \right)
  + \sum_{i = 1}^k \frac{\BF{n'}_{T, i}}{T} \log\left( \frac{\BF{q}_{T, i} - \frac{\BF{n'}_{T, i}}{T}}{p_i} \right) \nonumber\\
&= D_{\op{KL}} \left( \BF{q}_T \| \BF{q}_T - \frac{\BF{n'}_{T, i}}{T} \right)
  + \sum_{i = 1}^k \frac{\BF{n'}_{T, i}}{T} \log\left( \frac{\BF{q}_{T, i} - \frac{\BF{n'}_{T, i}}{T}}{p_i} \right).
\label{eq:lem:ASE:action_rKL:1}
\end{align}

Using inverse Pinsker's inequality, the first term may be bounded as
\begin{align*}
D_{\op{KL}} \left( \BF{q}_T \| \BF{q}_T - \frac{\BF{n'}_T}{T} \right)
&\leq \frac{2}{ \min_i \left\{ \BF{q}_{T, i} - \frac{\BF{n'}_{T, i}}{T} \right\} } d^{\op{TV}} \left( \BF{q}_T, \BF{q}_T - \frac{\BF{n'}_T}{T} \right)^2 \\
&\leq \frac{2}{ \min_i \left\{ \BF{q}_{T, i} - \frac{\BF{n'}_{T, i}}{T} \right\} } \left( \frac{1}{2 T} \sum_{i = 1}^k \left| \BF{n'}_T \right| \right)^2 \\
&= \frac{2}{ \min_i \left\{ \frac{1}{T} \sum_{t = 1}^T \hat{\pi}_{t, i} \right\} } \left( \frac{1}{2 T} \sum_{i = 1}^k \left| \BF{n'}_T \right| \right)^2.
\end{align*}
If we assume that $\C{E}$ occurs, using Inequalities~\ref{eq:lem:ASE:base_TV:1} and~\ref{eq:lem:ASE:base_TV:2}, we may further bound it as
\begin{align}
D_{\op{KL}} \left( \BF{q}_T \| \BF{q}_T - \frac{\BF{n'}_T}{T} \right)
&\leq \frac{2}{ \min_i \left\{ \frac{1}{T} \sum_{t = 1}^T \hat{\pi}_{t, i} \right\} } \left( \frac{1}{2 T} \sum_{i = 1}^k \left| \BF{n'}_T \right| \right)^2 \nonumber\\
&\leq \frac{2}{ \exp(-2 \epsilon) \alpha} \left( \frac{1}{2 T} \sum_{i = 1}^k \left| \BF{n'}_T \right| \right)^2 \nonumber\\
&\leq \frac{2}{ \exp(-2 \epsilon) \alpha} \left( \frac{1}{2 T} \sum_{i = 1}^k \sqrt{2 T \log(1/\delta)} \right)^2 \nonumber\\
&= \frac{\exp(2 \epsilon) k^2 \log(1/\delta)}{\alpha T}.
\label{eq:lem:ASE:action_rKL:2}
\end{align}

On the other hand, using the fact that $| \log(x) | \leq \max\{ | x - 1 |, |1/x - 1| \}$ for all $x > 0$, we have
\begin{align*}
\sum_{i = 1}^k \frac{\BF{n'}_{T, i}}{T} \log\left( \frac{\BF{q}_{T, i} - \frac{\BF{n'}_{T, i}}{T}}{p_i} \right)
&= \sum_{i = 1}^k \frac{\BF{n'}_{T, i}}{T} \log\left( \frac{\frac{1}{T} \sum_{t = 1}^T \hat{\pi}_{t, i}}{p_i} \right) \\
&\leq \sum_{i = 1}^k \frac{|\BF{n'}_{T, i}|}{T} \max\left\{ 
    \left| \frac{\frac{1}{T} \sum_{t = 1}^T \hat{\pi}_{t, i}}{p_i} - 1 \right| 
    , 
    \left| \frac{p_i}{\frac{1}{T} \sum_{t = 1}^T \hat{\pi}_{t, i}} - 1 \right| 
\right\} \\
&= \sum_{i = 1}^k \frac{|\BF{n'}_{T, i}|}{T}  
    \left| 
        \frac{\frac{1}{T} \sum_{t = 1}^T \hat{\pi}_{t, i} - p_i}
             { \min\left\{ p_i, \frac{1}{T} \sum_{t = 1}^T \hat{\pi}_{t, i} \right\} } 
    \right| 
\end{align*}
If we assume that $\C{E}$ occurs, using Inequalities~\ref{eq:lem:ASE:base_TV:1} and~\ref{eq:lem:ASE:base_TV:2}, we may further bound it as
\begin{align*}
\sum_{i = 1}^k \frac{\BF{n'}_{T, i}}{T} \log\left( \frac{\BF{q}_{T, i} - \frac{\BF{n'}_{T, i}}{T}}{p_i} \right)
&\leq \sum_{i = 1}^k \frac{|\BF{n'}_{T, i}|}{T}  
    \left| 
        \frac{\frac{1}{T} \sum_{t = 1}^T \hat{\pi}_{t, i} - p_i}
             { \min\left\{ p_i, \frac{1}{T} \sum_{t = 1}^T \hat{\pi}_{t, i} \right\} } 
    \right| \\
&\leq \sum_{i = 1}^k \frac{|\BF{n'}_{T, i}|}{T}  
    \left| 
        \frac{\frac{1}{T} \sum_{t = 1}^T \hat{\pi}_{t, i} - p_i}
             { \min\left\{ p_i, \exp(-2 \epsilon) \alpha \right\} } 
    \right| \\
&\leq \sum_{i = 1}^k \frac{|\BF{n'}_{T, i}|}{ \exp(-2 \epsilon) \alpha T^2}  
        \sum_{t = 1}^T \left| \hat{\pi}_{t, i} - p_i \right| \\
&\leq \sum_{i = 1}^k \frac{|\BF{n'}_{T, i}|}{ \exp(-2 \epsilon) \alpha T^2}  
        \sum_{t = 1}^T d^{\op{TV}}\left( \hat{\pi}_t, p_i \right) \\
&\leq \sum_{i = 1}^k \frac{|\BF{n'}_{T, i}|}{ \exp(-2 \epsilon) \alpha T^2}  
        \sum_{t = 1}^T d^{\op{TV}}\left( \hat{\pi}_t, p_i \right) \\
&= \sum_{i = 1}^k \frac{|\BF{n'}_{T, i}|}{ \exp(-2 \epsilon) \alpha T^2} \C{CPR}_T^{\op{TV}} \\
&\leq \frac{k \sqrt{2 T \log(1/\delta)}}{ \exp(-2 \epsilon) \alpha T^2}
    \left( T_0 + 8 \sigma_* \exp(\epsilon) \sqrt{\frac{\log(1/\delta) T}{\alpha}} \right),
\end{align*}
where we used Inequalities~\ref{eq:n_prime_azuma} and~\ref{eq:lem:ASE:policy_TV:1} in the last inequality.

Thus, using Inequalities~\ref{eq:lem:ASE:action_rKL:1} and~\ref{eq:lem:ASE:action_rKL:2}, when $\C{E}$ occurs, we have
\begin{align*}
\C{SAR}_T^{\op{r-KL}} - \frac{1}{T} \C{CPR}_T^{\op{r-KL}}
&\leq D_{\op{KL}} \left( \BF{q}_T \| \BF{q}_T - \frac{\BF{n'}_{T, i}}{T} \right)
  + \sum_{i = 1}^k \frac{\BF{n'}_{T, i}}{T} \log\left( \frac{\BF{q}_{T, i} - \frac{\BF{n'}_{T, i}}{T}}{p_i} \right) \\
&\leq \frac{\exp(2 \epsilon) k^2 \log(1/\delta)}{\alpha T}
  + \frac{k \sqrt{2 T \log(1/\delta)}}{ \exp(-2 \epsilon) \alpha T^2}
    \left( 8 \sigma_* \exp(\epsilon) \sqrt{\frac{\log(1/\delta) T}{\alpha}} + T_0 \right) \\
&= \frac{\exp(2 \epsilon) k^2 \log(1/\delta)}{\alpha T}
  + \frac{8 \sqrt{2} k \sigma_* \exp(3 \epsilon) \log(1/\delta)}{ \alpha^{3/2} T}
  + \frac{k \exp(2 \epsilon) \sqrt{2 \log(1/\delta)}}{ \alpha T^{3/2}} T_0.
\end{align*}

Therefore, using Lemma~\ref{lem:ASE:policy_rKL}, we have
\begin{align*}
\B{E} \left[ \C{SAR}_T^{\op{r-KL}} \right]
&= \B{P}(\C{E}) \B{E} \left[ \C{SAR}_T^{\op{r-KL}} | \C{E} \right] + \B{P}(\C{E}^c) \B{E} \left[ \C{SAR}_T^{\op{r-KL}} | \C{E}^c \right] \\
&\leq \B{E} \left[ \C{SAR}_T^{\op{r-KL}} | \C{E} \right] + \frac{2}{\alpha}\B{P}(\C{E}^c) \\
&\leq \B{E} \left[ \C{SAR}_T^{\op{r-KL}} | \C{E} \right] + \frac{4 k T \delta}{\alpha} \\
&\leq \frac{1}{T} \B{E}\left[ \C{CPR}_T^{\op{r-KL}} \mid \C{E} \right]
  + \frac{\exp(2 \epsilon) k^2 \log(1/\delta)}{\alpha T} \\
  &\qquad\qquad+ \frac{8 \sqrt{2} k \sigma_* \exp(3 \epsilon) \log(1/\delta)}{ \alpha^{3/2} T}
  + \frac{k \exp(2 \epsilon) \sqrt{2 \log(1/\delta)}}{ \alpha T^{3/2}} T_0
  + \frac{4kT\delta}{\alpha} \\
&\leq \frac{1}{T} \left( 
    \frac{32 \sigma_*^2 \exp(2 \epsilon) \log(1/\delta) \log(T)}{\alpha^2}
    + \frac{2 T_0}{\alpha} 
  \right)
  + \frac{\exp(2 \epsilon) k^2 \log(1/\delta)}{\alpha T} \\
  &\qquad\qquad+ \frac{8 \sqrt{2} k \sigma_* \exp(3 \epsilon) \log(1/\delta)}{ \alpha^{3/2} T}
  + \frac{k \exp(2 \epsilon) \sqrt{2 \log(1/\delta)}}{ \alpha T^{3/2}} T_0
  + \frac{4kT\delta}{\alpha}.
\qedhere
\end{align*}
\end{proof}

\begin{lemma}[Policy-level forward-KL regret]\label{lem:ASE:policy_fKL}
If we further assume that $\nu \in \C{E}_{\op{B}}^k(\sigma)$, $C = 0$ and $T_0 \leq t \leq T$, then we have
\begin{align*}
\B{E} \left[ \C{SPR}_t^{\op{f-KL}} \right]
&\leq \frac{32 \sigma^2 \exp(4 \epsilon) \log(1/\delta)}{\alpha^2 t}
  + \frac{4 \exp(2 \sigma) k T \delta}{\alpha}, \\
\B{E} \left[ \C{CPR}_T^{\op{f-KL}} \right]
&\leq \frac{32 \sigma^2 \exp(4 \epsilon) \log(1/\delta) \log(T)}{\alpha^2}
  + \frac{2 \exp(2 \sigma) T_0}{\alpha} 
  + \frac{4 \exp(2 \sigma) k T^2 \delta}{\alpha},
\end{align*}
where cumulative regret is defined as the sum from time-step $Mk + 1$, as discussed in Remark~\ref{rem:policy-fkl}.
Thus,
$\B{E} \left[ \C{SPR}_T^{\op{f-KL}} \right] = \tilde{O}(\exp(4 \epsilon) T^{-1})$
and
$\B{E} \left[ \C{CPR}_T^{\op{f-KL}} \right] = \tilde{O}(\exp(4 \epsilon))$.
\end{lemma}

\begin{proof}
Recall that $\C{E}_{\op{B}}^k(\sigma) \subseteq \C{E}_{\op{SG}}^k(\sigma^2)$ by Hoeffding's lemma and therefore all the previous lemmas hold.
Using Theorem~\ref{thm:KL-vs-TV} and Lemma~\ref{lem:ASE:base_TV} we see that
\begin{align*}
\B{E} \left[ \C{SPR}_t^{\op{f-KL}} | \C{E} \right]
&= \B{E} \left[ \frac{2}{\min_i \hat{\BF{p}}_{t, i}} \left( \C{SPR}_t^{\op{TV}} \right)^2 | \C{E} \right] \\
&= \B{E} \left[ \frac{2}{\min_i \hat{\BF{p}}_{t, i}} \left( d^{\op{TV}}(\hat{\BF{p}}_t, p) \right)^2 | \C{E} \right] \\
&\leq \B{E} \left[ \frac{2}{\exp(-2 \epsilon) \alpha} \left( d^{\op{TV}}(\hat{\BF{p}}_t, p) \right)^2 | \C{E} \right] \\
&\leq \frac{32 \sigma^2 \exp(4 \epsilon) \log(1/\delta)}{\alpha^2 t}.
\end{align*}
Since $\nu \in \C{E}_{\op{B}}^k(\sigma)$, all observation are within $\sigma$ of the true reward for all arms.
Thus, since $C = 0$, we have
\begin{align*}
| \hat{\BF{r}}_{t, i} - r_i |
= \left| \frac{1}{\BF{n}_{t, i}} \sum_{t'=1}^t \BF{x}_{t'} \BF{1}_{\BF{a}_{t'} = i} - r_i \right|
\leq \sigma.
\end{align*}
Hence we may use Lemma~\ref{lem:bound-on-softmax} to see that $\hat{p}_{t', i} \geq p_i \exp(-2\sigma)$.
Therefore we have
\begin{align*}
\C{SPR}_{t'}^{\op{f-KL}} 
= D_{\op{KL}}(p \| \hat{\pi}_{t'}) 
\leq \frac{2}{\min_i \hat{\pi}_{t', i}} 
= \frac{2}{\min_i \hat{\BF{p}}_{t', i}} 
\leq \frac{2}{\exp(-2 \sigma) \alpha},
\end{align*}
for all $Mk + 1 \leq t' \leq T$.
Thus
\begin{align*}
\B{E} \left[ \C{SPR}_t^{\op{f-KL}} \right]
&= \B{P}(\C{E}) \B{E} \left[ \C{SPR}_t^{\op{f-KL}} | \C{E} \right] 
  + \B{P}(\C{E}^c) \B{E} \left[ \C{SPR}_t^{\op{f-KL}} | \C{E}^c \right] \\
&\leq \B{E} \left[ \C{SPR}_t^{\op{f-KL}} | \C{E} \right] + \frac{2 \exp(2 \sigma)}{\alpha}\B{P}(\C{E}^c) \\
&\leq \frac{32 \sigma^2 \exp(4 \epsilon) \log(1/\delta)}{\alpha^2 t}
  + \frac{4 \exp(2 \sigma) k T \delta}{\alpha}.
\end{align*}
To bound the cumulative action-level regret, we note that
\begin{align*}
\B{E} \left[ \C{CPR}_T^{\op{f-KL}} | \C{E} \right]
&= \sum_{t = Mk + 1}^T \B{E} \left[ \C{SPR}_t^{\op{f-KL}} | \C{E} \right] \\
&= \sum_{t = Mk + 1}^{T_0 - 1} \B{E} \left[ \C{SDR}_t^{\op{f-KL}} | \C{E} \right] + \sum_{t = T_0}^T \B{E} \left[ \C{SDR}_t^{\op{f-KL}} | \C{E} \right] \\
&\leq \frac{2 \exp(2 \sigma) T_0}{\alpha}
  + \sum_{t = T_0}^T \left( \frac{32 \sigma^2 \exp(4 \epsilon) \log(1/\delta)}{\alpha^2 t} \right) \\
&\leq \frac{2 \exp(2 \sigma) T_0}{\alpha} 
  + \frac{32 \sigma^2 \exp(4 \epsilon) \log(1/\delta) \log(T)}{\alpha^2}.
\end{align*}
Therefore
\begin{align*}
\B{E} \left[ \C{CPR}_T^{\op{f-KL}} \right]
&= \B{P}(\C{E}) \B{E} \left[ \C{CPR}_T^{\op{f-KL}} | \C{E} \right] 
  + \B{P}(\C{E}^c) \B{E} \left[ \C{CPR}_T^{\op{f-KL}} | \C{E}^c \right] \\
&\leq \B{E} \left[ \C{CPR}_T^{\op{f-KL}} | \C{E} \right] 
  + \frac{2 \exp(2 \sigma) T}{\alpha}\B{P}(\C{E}^c) \\
&\leq \frac{32 \sigma^2 \exp(4 \epsilon) \log(1/\delta) \log(T)}{\alpha^2}
  + \frac{2 \exp(2 \sigma) T_0}{\alpha} 
  + \frac{4 \exp(2 \sigma) k T^2 \delta}{\alpha}.
\qedhere
\end{align*}
\end{proof}

\begin{lemma}[Action-level forward-KL regret]\label{lem:ASE:action_fKL}
Under the assumptions of Lemma~\ref{lem:ASE:policy_fKL}, we have
\begin{align*}
\B{E} \left[ \C{SAR}_T^{\op{f-KL}} \right]
&= \tilde{O}(\exp(6 \epsilon) T^{-1}), 
\quad\t{ and }\quad
\B{E} \left[ \C{CAR}_T^{\op{f-KL}} \right]
= \tilde{O}(\exp(6 \epsilon)).
\end{align*}
\end{lemma}

\begin{proof}
The proof is similar to the previous cases.
\end{proof}

\section{PROOF OF THEOREM~\ref{thm:MAS_to_MAB}}\label{apx:MAS_to_MAB}

\begin{proof}
For $0 < \beta < \infty$, let $Z^\beta = \sum_{j = 1}^k e^{\beta r_j}$.
It follows that 
$p^\beta = \op{softmax}(\beta r) = \frac{e^{\beta r_i}}{Z^\beta}$.
Therefore, if we define $r_* = \max_{1 \leq i \leq k} r_i$ and $p^\beta_* = \max_{1 \leq i \leq k} p^\beta_i$, then we have
$\log(p^\beta_*/p^\beta_i) = \beta (r_* - r_i)$
for $1 \leq i \leq k$.
Hence
\begin{align*}
d^{\op{r-KL}}(\BF{q}_T, p^\beta)
&= D_{\op{KL}}(\BF{q}_T \| p^\beta) 
= \sum_{i = 1}^k \BF{q}_{T, i} \log\left( \frac{ \BF{q}_{T, i} }{ p^\beta_i } \right) \\
&= \sum_{i = 1}^k \BF{q}_{T, i} \log\left( \frac{p^\beta_*}{p^\beta_i} \right) 
  - \log\left( p^\beta_* \right)
  + \sum_{i = 1}^k \BF{q}_{T, i} \log\left( \BF{q}_{T, i} \right) \\
&= \frac{\beta}{T} \sum_{i = 1}^k \BF{n}_{T, i} (r_* - r_i) - \log(p^\beta_*) - H(\BF{q}_T) \\
&= \frac{\beta}{T} \C{R}_T^{\op{MAB}} - \log(p^\beta_*) - H(\BF{q}_T).
\end{align*}
Similarly
\begin{align*}
d^{\op{r-KL}}(\hat{\pi}_t, p^\beta)
&= D_{\op{KL}}(\hat{\pi}_t \| p^\beta) 
= \sum_{i = 1}^k \hat{\pi}_{t, i} \log\left( \frac{ \hat{\pi}_{t, i} }{ p^\beta_i } \right) \\
&= \sum_{i = 1}^k \hat{\pi}_{t, i} \log\left( \frac{p^\beta_*}{p^\beta_i} \right) 
  - \log\left( p^\beta_* \right) 
  + \sum_{i = 1}^k \hat{\pi}_{t, i} \log\left( \hat{\pi}_{t, i} \right) \\
&= \beta \sum_{i = 1}^k \hat{\pi}_{t, i} (r_* - r_i) - \log(p^\beta_*) - H(\hat{\pi}_t) \\
&= \beta \B{E} [ \C{SR}_t^{\op{MAB}} | \BF{h}_t ] - \log(p^\beta_*) - H(\hat{\pi}_t).
\end{align*}
Thus
\begin{align*}
\B{E} [ d^{\op{r-KL}}(\hat{\pi}_t, p^\beta) ]
&= \B{E} [ \beta \C{SR}_t^{\op{MAB}} - H(\hat{\pi}_t) ] - \log(p^\beta_*),
\end{align*}
which implies that
\begin{align*}
\B{E} \left[ \sum_{t = 1}^T d^{\op{r-KL}}(\hat{\pi}_t, p^\beta) \right]
&= \sum_{t = 1}^T \B{E} \left[ \beta \C{SR}_t^{\op{MAB}} - H(\hat{\pi}_t) \right] - T \log(p^\beta_*).
\end{align*}

The entropy terms $H(\BF{q}_T)$ and $H(\hat{\pi}_T)$ are always bounded by $0$ and $\log(k)$ and $\log(p^\beta_*) = \max_i \log(p^\beta_i) \geq \log(1/k)$ is bounded between $-\log(k)$ and $0$.
It follows that 
\begin{align*}
\lim_{\beta \to \infty} \frac{T}{\beta} d^{\op{r-KL}}(\BF{q}_T, p^\beta)
&= \C{R}_T^{\op{MAB}}, \\
\lim_{\beta \to \infty} \frac{1}{\beta} \B{E}[ d^{\op{r-KL}}(\hat{\pi}_t, p^\beta) ]
&= \B{E}[ \C{SR}_t^{\op{MAB}} ], \\
\lim_{\beta \to \infty} \frac{1}{\beta} \sum_{t = 1}^T \B{E}[ d^{\op{r-KL}}(\hat{\pi}_t, p^\beta) ]
&= \B{E}[ \C{R}_T^{\op{MAB}} ].
\qedhere
\end{align*}
\end{proof}

\section{MINIMIZING \texorpdfstring{$\beta$}{BETA}-REGRET}\label{apx:alg-beta}

If $\nu$ is an environment and $c \in \B{R}$, we define $c \nu$ to be the environment where each observation and reward is $c$ times the corresponding value in $\nu$.
On the other hand, for any policy $\pi$, we define $\pi^c$ to be the policy that first multiplies any received observation by $c$ and then applies $\pi$.

\begin{theorem}\label{thm:beta-regret}
For any policy $\pi$, environment $\nu$, and $0 < \beta < \infty$, we have
\begin{align*}
\C{R}_T^\beta(\pi^\beta, \nu) 
= \frac{1}{\beta} \C{R}_T^1(\pi, \beta \nu).
\end{align*}
\end{theorem}
To see why this holds, we simply note that the game played between $\pi^\beta$ and $\nu$ is identical to one played between $\pi$ and $\beta \nu$.
We only need to divide by $\beta$ to account for the effect of division by $\beta$ in the definition of $\beta$-regret.

Using this theorem, we see that the regret bounds in Theorem~\ref{thm:AS} could be used to obtain $\beta$-regret bounds for $\pi^\beta$ for all $0 < \beta < \infty$.
Thus demonstrating that $\pi^\beta$ is optimal for all such $\beta$.

At this point, we may see a contradiction by putting together the fact that (1) sampling does not require exploration and (2)
sampling from the zero temperature limit of a distribution corresponds to optimization; therefore one should be able to optimize without exploration, a conclusion that contradicts 
the fact that exploration is necessary for multi-armed bandits. 
This apparent contradiction is resolved by noting that optimality is measured only in terms of time and other constants are ignored.
When the temperature goes to zero, i.e., $\beta$ goes to infinity, the role of these constants slowly grows and in the limit, they dominate. 
Thus resulting in linear regret of Algorithm~\ref{alg:AS} in the bandit setting.

To make this claim more precise, first, we note that $\pi^\infty$ is effectively Explore-Then-Commit; see \cref{sec:related:MAB}.
The reason for this claim is that $\pi^\beta$ is equivalent to running $\pi$, but replacing the line $\hat{p}_{t+1} \gets \op{softmax}(\hat{r}_t)$ with $\hat{p}_{t+1} \gets \op{softmax}(\beta \hat{r}_t)$.
In the limit of $\beta \to \infty$, the expression $\op{softmax}(\beta \hat{r}_t)$ tends to the distribution that has equal weights on arms with $\hat{r}_{t, i} = \op{max}_j \hat{r}_{t, j}$ and zero weight elsewhere.
Thus $\pi^\beta$, in some sense, tends to Explore-Then-Commit.
\footnote{
In ETC algorithm, the estimates for arms are not updated after the exploration phase.
This is not the case in this limiting algorithm, where the estimated means will be updated at every time-step.
Thus, it would be more accurate to say that the limiting algorithm is a small variant of the classical ETC algorithm.
However, this variation does not have a significant effect on the performance of the algorithm.
Specifically, it is trivial to see that just like the classical ETC, without significant exploration, this algorithm still suffers a linear regret.
}
Now, if we take the limit of both sides in Theorem~\ref{thm:beta-regret}, we see that 
\begin{align*}
\C{R}^{\op{MAB}}_T(\op{ETC}, \nu)
= \lim_{\beta \to \infty} \C{R}^\beta_T(\pi^\beta, \nu) 
= \lim_{\beta \to \infty} \frac{1}{\beta} \C{R}^1_T(\pi, \beta \nu).
\end{align*}
Note that this is not a rigorous mathematical statement as we have not formally defined the notion of limit for algorithms.
The expression of regret bound in Lemma~\ref{lem:ASE:action_rKL} contains two terms that are affected by replacing $\nu$ with $\beta \nu$.
The variance $\sigma^2$ gets multiplied by $\beta^2$ and the term $\alpha = \min_i p_i^\beta$ goes to zero exponentially fast with $\beta$.
Thus we see that, even though the dependence of the regret bound on time is small, its dependence on $\beta$ results in vacuous regret bound in the multi-armed bandit limit.

\section{EXPERIMENTS}\label{apx:experiments}

In Figure~\ref{fig:apx:action}, we plot action-level regret of GAS, GASW (with logarithmic exploration), DAISEE, and KL-UCB.
In the first row we plot simple action-level regret multiplied by time-step to make the plots easier to see.
The second row demonstrates the cumulative action-level regret.
From left to right, the distances are reverse-KL, forward-KL and total variation.
The initial increase in regret in the orange lines corresponds to the exploration phase (and its effect on action history).

In Figure~\ref{fig:apx:policy}, we plot policy-level regret of the same algorithms.
Similar to the previous figure, the first row we plot simple policy-level regret multiplied by time-step to make the plots easier to see.
Here the effect of the exploration phase is more pronounced.
Thus, for simple regret that does not rely on history and only depends on the current time-step, we assign zero to the GASW algorithm during the exploration phase.
We include the effect of the exploration phase in the cumulative regret.

Following the discussion in Remark~\ref{rem:policy-fkl}, we have assigned zero regret to the exploration phase for forward-KL policy-level regret, corresponding to the second column in Figure~\ref{fig:apx:policy}.
This hides the effect of the exploration phase for this choice of statistical distance.

In these experiments, there are 10 arms, the environments have unit normal noise and the reward are sampled from $[0, 1]$.
The average over 10 runs are shown and one standard error is highlighted.

\begin{figure*}
\centering
\hbox{
\includegraphics[width=\textwidth/3]{plots/simple_action-level_reverse-KL_regret.pdf}

\includegraphics[width=\textwidth/3]{plots/simple_action-level_forward-KL_regret.pdf}

\includegraphics[width=\textwidth/3]{plots/simple_action-level_total-variation_regret.pdf}
}
\hbox{
\includegraphics[width=\textwidth/3]{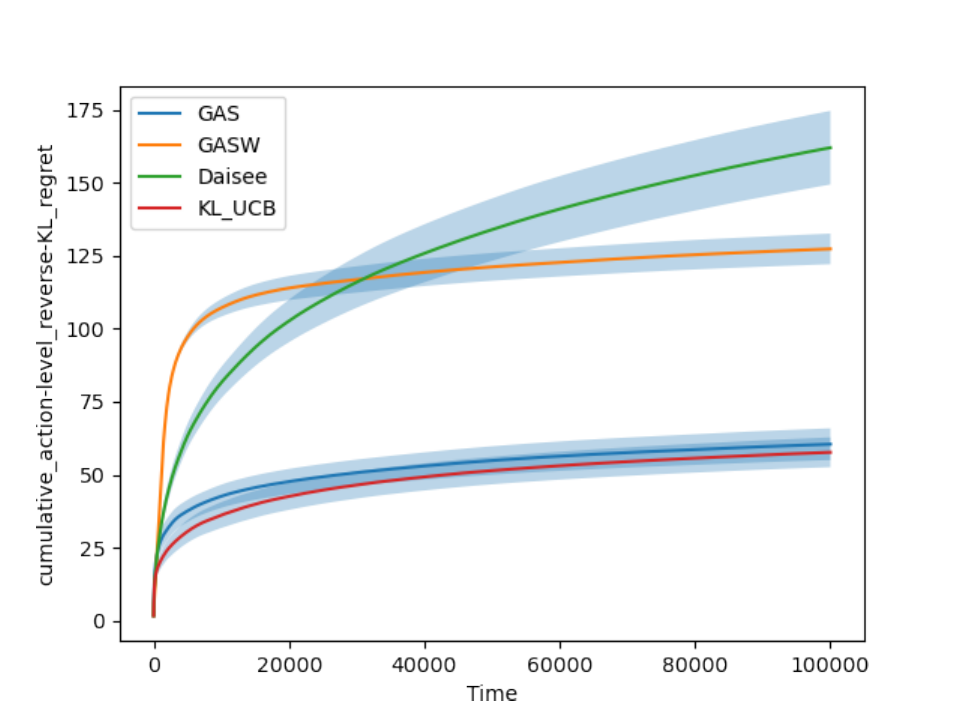}

\includegraphics[width=\textwidth/3]{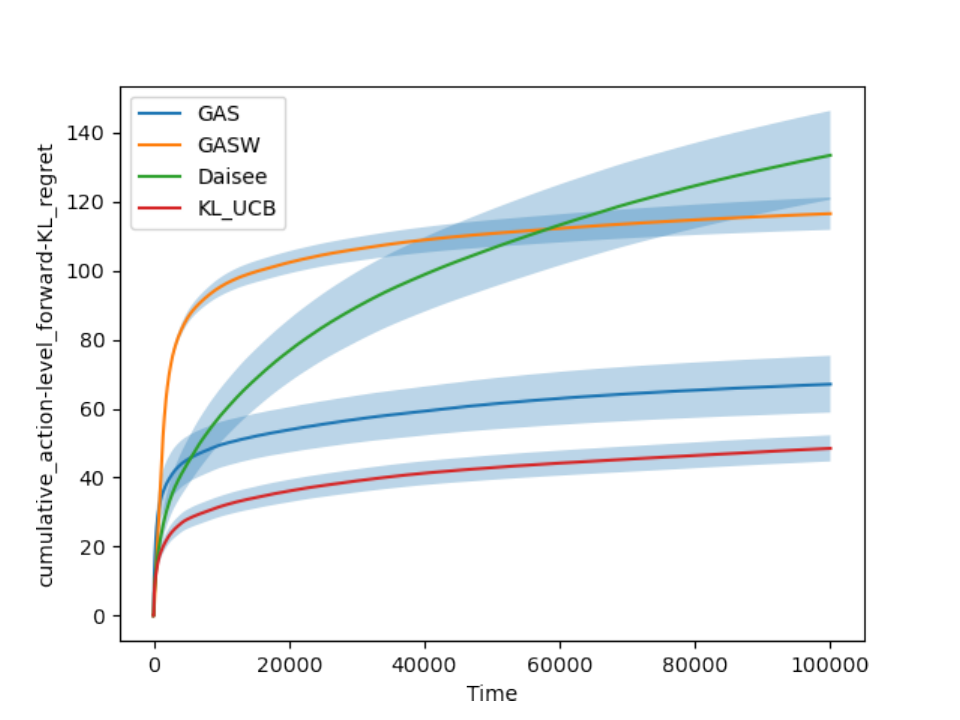}

\includegraphics[width=\textwidth/3]{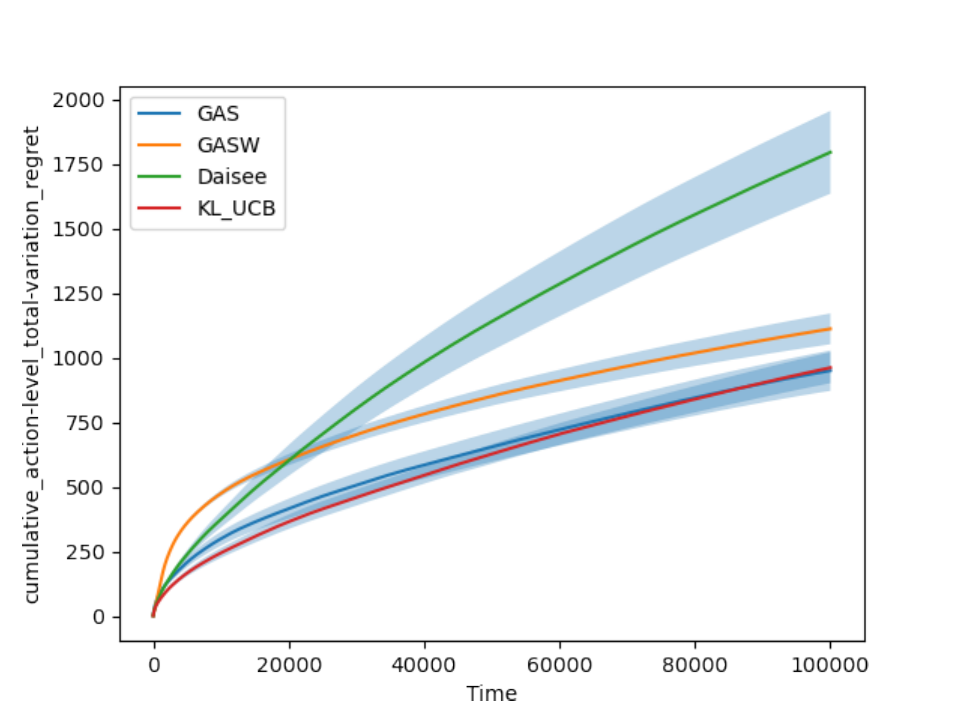}
}
\caption{
Action-level regret
}
\label{fig:apx:action}
\vspace{-0.5cm}
\end{figure*}

\begin{figure*}
\centering
\hbox{
\includegraphics[width=\textwidth/3]{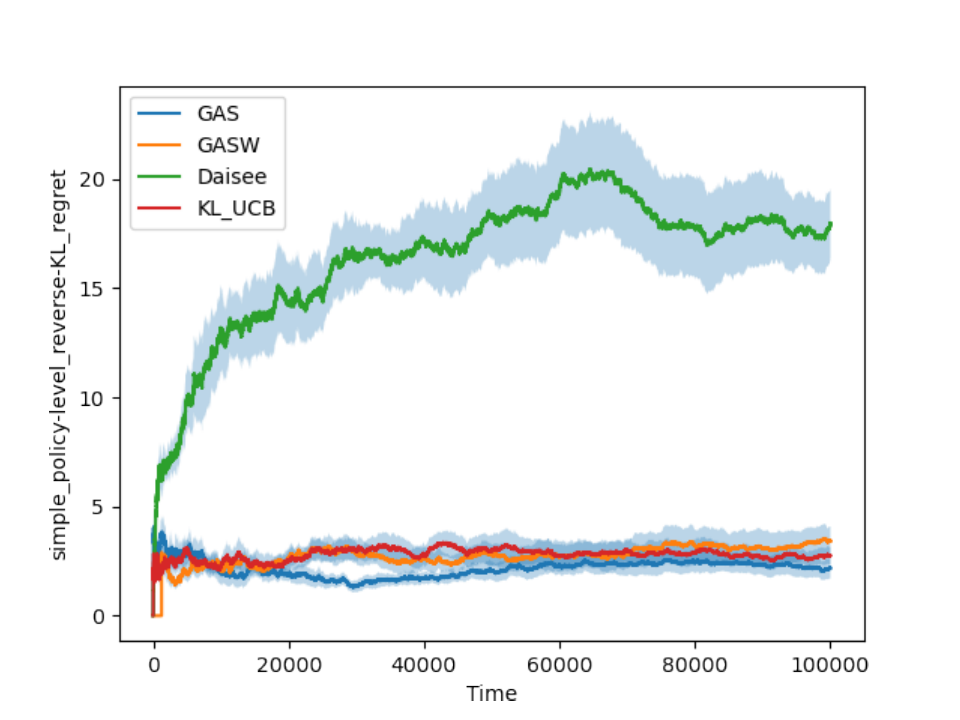}

\includegraphics[width=\textwidth/3]{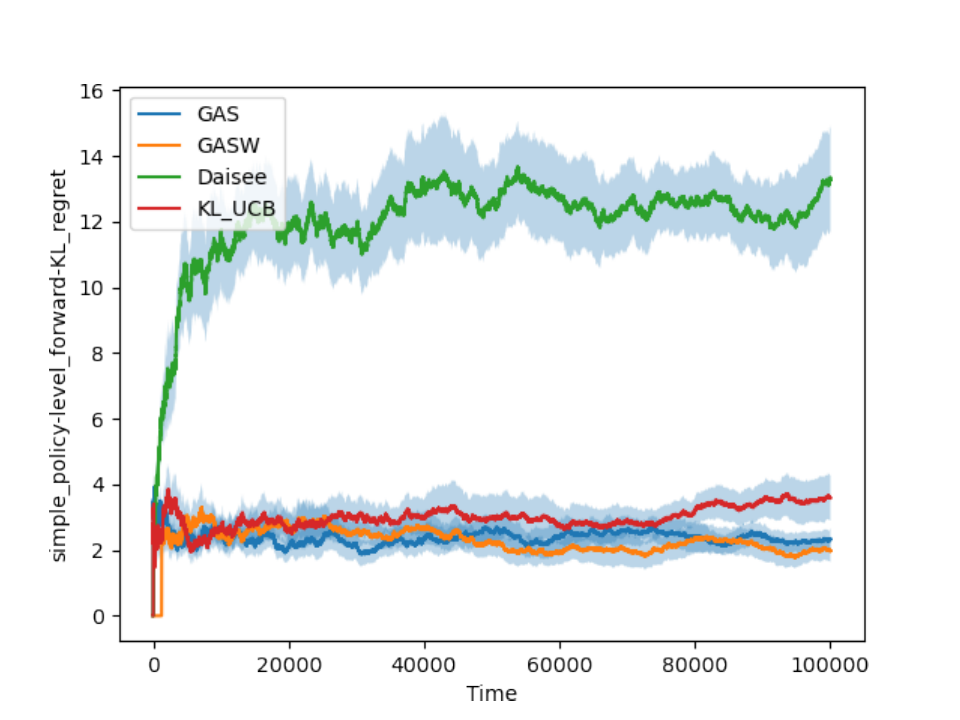}

\includegraphics[width=\textwidth/3]{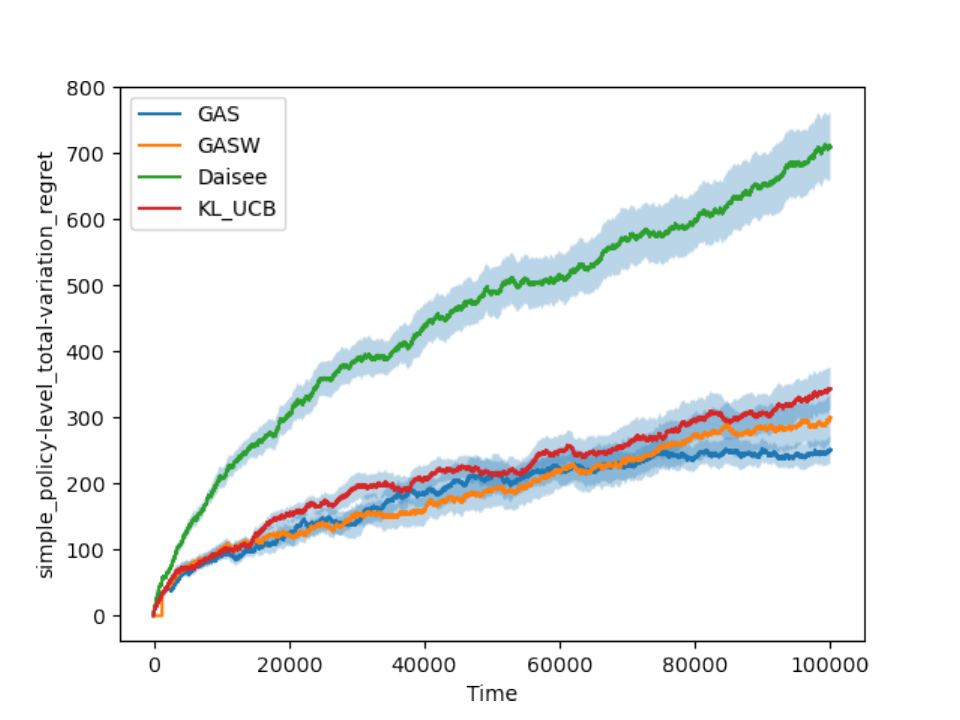}
}
\hbox{
\includegraphics[width=\textwidth/3]{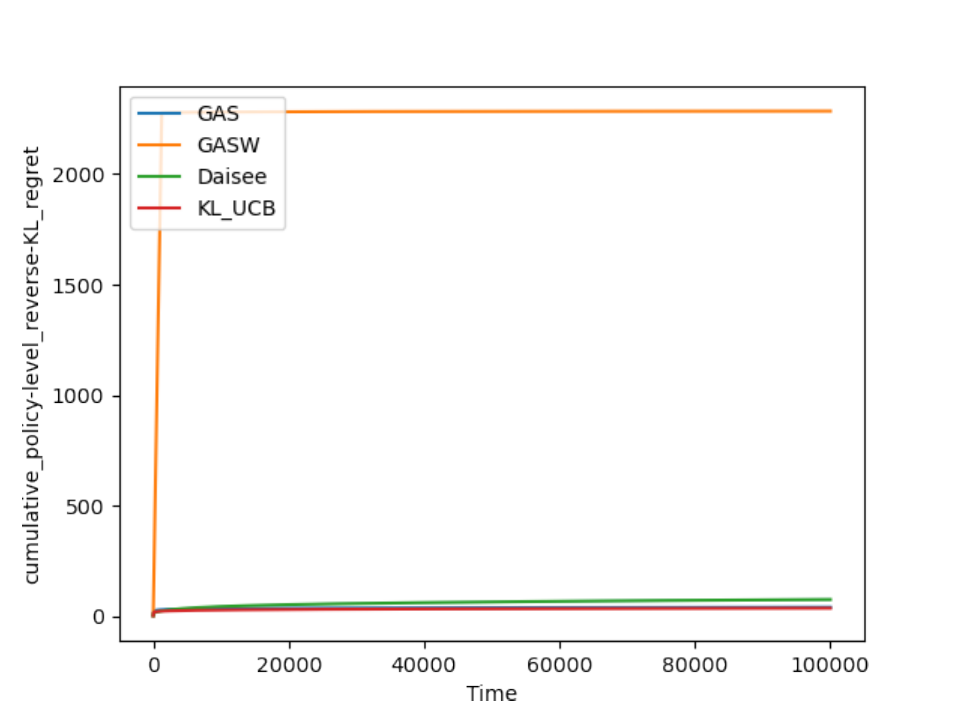}

\includegraphics[width=\textwidth/3]{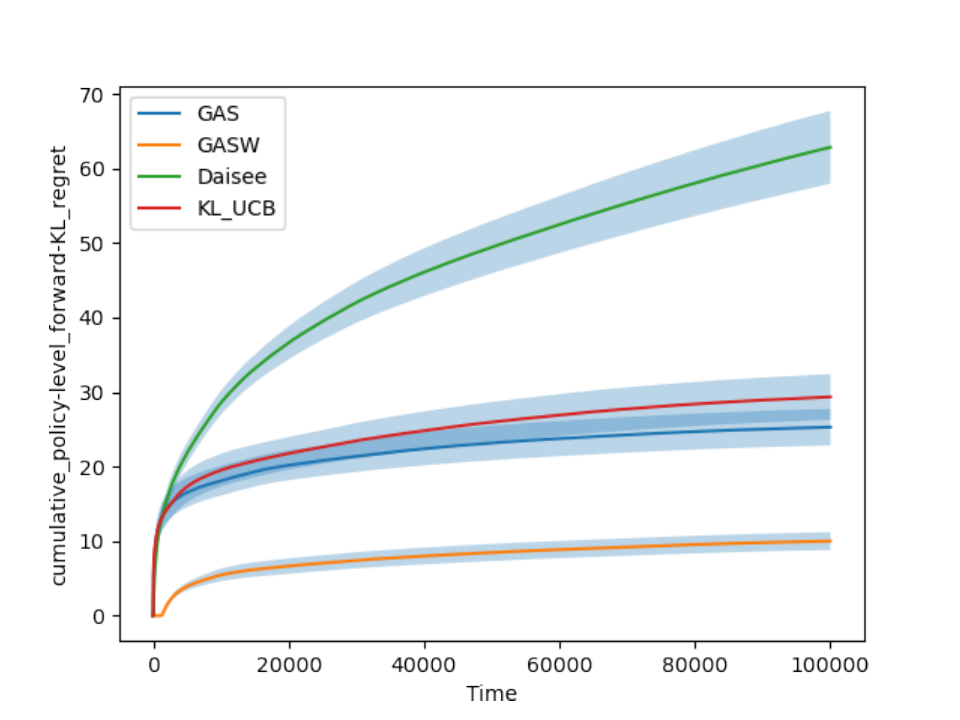}

\includegraphics[width=\textwidth/3]{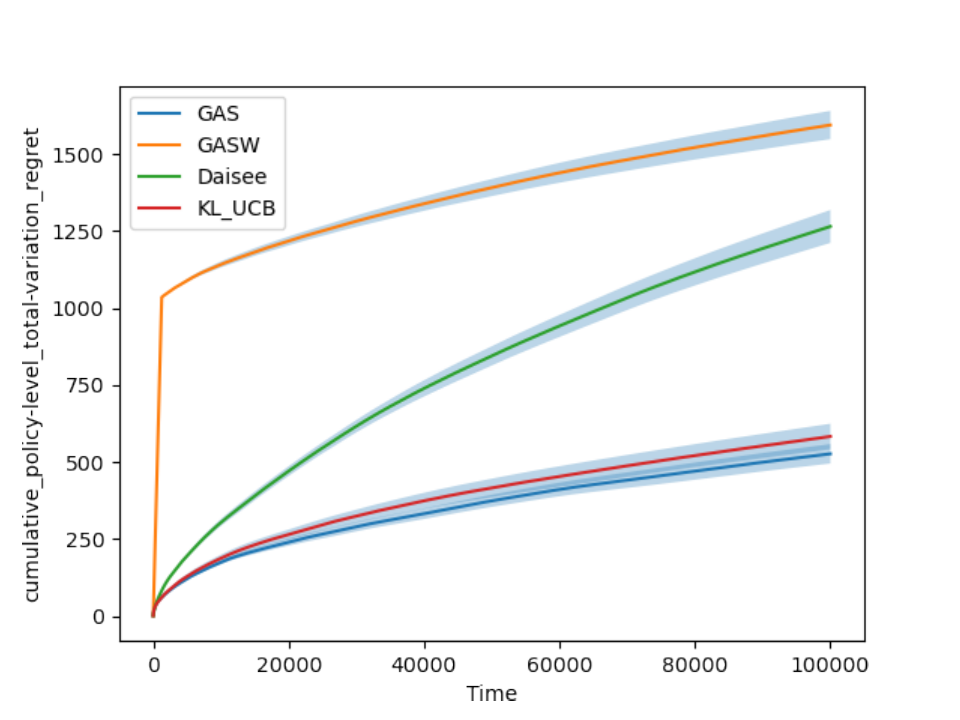}
}
\caption{
Policy-level regret
}
\label{fig:apx:policy}
\vspace{-0.5cm}
\end{figure*}

\end{document}